\documentclass{article}

\newcommand\independent{\protect\mathpalette{\protect\independenT}{\perp}}
\def\independenT#1#2{\mathrel{\rlap{$#1#2$}\mkern2mu{#1#2}}}
\newcommand\norm[1]{\left\lVert#1\right\rVert}

\usepackage{microtype}
\usepackage{graphicx}
\usepackage{scalerel}
\usepackage{enumitem}
\usepackage{multirow}
\usepackage{booktabs} 
\usepackage[justification=centering]{caption}
\usepackage{subcaption}
\usepackage{bbm}

\usepackage{hyperref}

\usepackage[accepted]{icml2023}

\usepackage{amsmath}
\usepackage{amssymb}
\usepackage{mathtools}
\usepackage{amsthm}

\usepackage[capitalize,noabbrev]{cleveref}

\theoremstyle{plain}
\newtheorem{theorem}{Theorem}[section]

\theoremstyle{definition}

\theoremstyle{remark}

\usepackage[textsize=tiny]{todonotes}

\icmltitlerunning{Estimating Causal Effects using a Multi-task Deep Ensemble}

\begin{document}

\twocolumn[
\icmltitle{Estimating Causal Effects using a Multi-task Deep Ensemble}



\icmlsetsymbol{equal}{*}

\begin{icmlauthorlist}
\icmlauthor{Ziyang Jiang}{duke_civil}
\icmlauthor{Zhuoran Hou}{duke_biostats}
\icmlauthor{Yiling Liu}{duke_comb}
\icmlauthor{Yiman Ren}{umich}
\icmlauthor{Keyu Li}{duke_ee}
\icmlauthor{David Carlson}{duke_civil,duke_biostats,duke_ee,duke_cs}
\end{icmlauthorlist}

\icmlaffiliation{duke_civil}{Department of Civil and Environmental Engineering, Duke University, Durham, NC, USA}
\icmlaffiliation{duke_biostats}{Department of Biostatistics and Bioinformatics, Duke University School of Medicine, Durham, NC, USA}
\icmlaffiliation{duke_comb}{Program in Computational Biology and Bioinformatics, Duke University School of Medicine, Durham, NC, USA}
\icmlaffiliation{duke_ee}{Department of Electrical and Computer Engineering, Duke University, Durham, NC, USA}
\icmlaffiliation{umich}{Department of Economics, University of Michigan Ross School of Business, Ann Arbor, MI, USA}
\icmlaffiliation{duke_cs}{Department of Computer Science, Duke University, Durham, NC, USA}

\icmlcorrespondingauthor{David Carlson}{david.carlson@duke.edu}

\icmlkeywords{Machine Learning, ICML}

\vskip 0.3in
]



\printAffiliationsAndNotice{}  

\begin{abstract}
A number of methods have been proposed for causal effect estimation, yet few have demonstrated efficacy in handling data with complex structures, such as images. To fill this gap, we propose Causal Multi-task Deep Ensemble (CMDE), a novel framework that learns both shared and group-specific information from the study population. We provide proofs demonstrating equivalency of CDME to a multi-task Gaussian process (GP) with a coregionalization kernel \emph{a priori}. Compared to multi-task GP, CMDE efficiently handles high-dimensional and multi-modal covariates and provides pointwise uncertainty estimates of causal effects. We evaluate our method across various types of datasets and tasks and find that CMDE outperforms state-of-the-art methods on a majority of these tasks.
\end{abstract}

\section{Introduction}
\label{sec:1}
Estimating the causal effect of an action is a fundamental step in determining whether it is significant enough to change human behavior in real-world settings. This is commonly used to help us understand the potential impact of an action and to inform decision-making. For example, governments often judge the impact of an implemented policy by gauging public opinion regarding said policy \cite{page1983effects}. Researchers can assess the efficacy and risk of a medical procedure through the use of both clinical trials, which examine the medical conditions of patients both with and without having the treatment, and observational data, such as electronic health records \cite{jensen2012mining,zhang2019medical}. In recent years, people have come up with a variety of approaches that leverage machine learning models to discover causal relationships \cite{guo2020survey,li2022survey}. Such methods include, for example, Targeted Maximum Likelihood Estimator \cite{van2006targeted},  Bayesian Additive Regression Trees \cite{chipman2010bart}, and Double/Debiased Machine Learning methods \cite{chernozhukov2018double}. Many recent studies focus on learning the individualized treatment effect (ITE) or conditional average treatment effect (CATE) by using deep learning models (see a more comprehensive review in Section \ref{sec:4}) or meta-learners \cite{kunzel2019metalearners}.

As technology progresses, an increasing number of datasets containing more complex and comprehensive infromation have become available for causal analysis. These datasets often include high-dimensional covariates, such as images, posing unique challenges for causal inference. While previous methods have demonstrated promising performance on a variety of causal inference tasks, including the Infant Health and Development (IHDP \cite{brooks1992effects}), Twins \cite{almond2005costs}, and Jobs \cite{lalonde1986evaluating}, few have been specifically evaluated on datasets with high-dimensional and multi-modal structures. In some situations, these complex covariates play a significant role in causal analysis. For instance, neuroimaging is essential in studying how human brain solves problems with multi-sensory causal inference \cite{kayser2015multisensory}, and brain imaging has been used to forecast treatment response in depression \cite{drysdale2017resting}, showing that information about causal relationships are embedded in complex data types. This highlights the need for further research on methods that can effectively handle high-dimensional and multi-modal covariates in causal analysis. In this paper, we propose a deep learning framework to address this challenge. Our main contributions are summarized as follows: 
\begin{itemize}
\item We propose the Causal Multi-task Deep Ensemble (CMDE) framework which estimates the CATE by learning both shared and group-specific information from control and treatment groups in the study population using separate neural networks.
\item We demonstrate the relationship between CMDE and multi-task Gaussian process (GP) framework \cite{alaa2017bayesian} both through analytical proof and empirical evaluation.
\item We propose an alternative configuration of CMDE to handle covariates in a multi-modal setting (e.g., images and tabular data).
\item By conducting experiments on semi-synthetic and real-world datasets containing covariates with various number of dimensions and modalities, we show that CMDE outperforms the state-of-the-art methods by improving estimation of the treatment effects.
\end{itemize}


\section{Background}
\label{sec:2}
\subsection{Problem Setup}
\label{sec:2.1}
We will use lower-case letters (e.g., $x, t, y$) for individual samples, upper-case letters (e.g., $X, T, Y$) for random variables, and bold upper-case letters (e.g., $\boldsymbol{X}, \boldsymbol{T}, \boldsymbol{Y}$) for a set of samples throughout the paper. We consider a general setting in causal inference where we assign a specific treatment to a group of individuals. Each individual is represented by a $D$-dimensional feature vector $X \in \mathcal{X} \subset \mathbb{R}^D$ ($\mathcal{X}$ denotes the training input space) and is associated with a treatment-assignment indicator $T \in \{0, 1\}$. The corresponding \emph{potential outcomes} are denoted by $Y^{(0)} \in \mathbb{R}$ and $Y^{(1)} \in \mathbb{R}$ where the superscripts $0$ and $1$ represent assignment to the control group and the treatment group, respectively. We assume there exist a joint distribution $P\left(X, T, Y^{(0)}, Y^{(1)}\right)$ which satisfies $0 < P(T=1|X) < 1$ and the strong ignorability assumption $\left(Y^{(0)}, Y^{(1)}\right) \independent T|X$ as given in the Rubin-Neyman causal model \cite{rosenbaum1983central,rubin2005causal}. Our goal is to estimate the \emph{conditional average treatment effect} (CATE) from a training dataset containing $N$ data points $\mathcal{D} = \left\{ x_i, t_i, y_i^{(t)} \right\}_{i=1}^{N}$ 
where CATE can be computed as $\text{CATE} := \mathbb{E} \left[Y^{(1)} - Y^{(0)}|X\right]$. We denote $Y^{(T)}$ and $Y^{(1-T)}$ as \emph{factual} and \emph{counterfactual} outcomes, respectively. That is, we have $Y^{(T)} = (1-T)Y^{(0)} + TY^{(1)}$ and $Y^{(1-T)} = TY^{(0)} + (1-T)Y^{(1)}$.

\subsection{Multi-task Gaussian Processes (GPs)}
\label{sec:2.2}
We first introduce the background of multi-task GPs. A GP is a stochastic process that is completely defined by a mean function $\mu: \mathcal{X} \rightarrow \mathbb{R}$ and a kernel function $k: \mathcal{X} \times \mathcal{X} \rightarrow \mathbb{R}$ \cite{rasmussen2003gaussian}. Without loss of generality, we assume that the mean function $\mu(\boldsymbol{x})$ is zero for simplicity, as is common in the literature. A single-output function $f: \mathcal{X} \rightarrow \mathbb{R}$ following a GP is written as
\begin{equation}
f \sim \mathcal{GP} \left( 0, k \right).
\end{equation}
For any finite subset $\boldsymbol{X} = \{x_i\}_{i=1}^N$, $f(\boldsymbol{X})$ follows a multivariate Gaussian distribution with mean zero and covariance matrix $k(\boldsymbol{X}, \boldsymbol{X}) \in \mathbb{R}^{N \times N}$ with entries $k(x_i, x_j)$ where $1 \leq i, j \leq N$. We extend a GP to a multi-task learning scenario by defining a \emph{vector-valued} function $\textbf{f}: \mathcal{X} \rightarrow \mathbb{R}^C$, and we can write the corresponding multi-task GP as
\begin{equation}
\textbf{f} \sim \mathcal{GP} \left( \boldsymbol{0}, \textbf{K} \right),
\end{equation}
where $\textbf{K}: \mathcal{X} \times \mathcal{X} \rightarrow \mathbb{R}^{C \times C}$ denotes a \emph{matrix-valued} kernel function. Again, any finite subset $\textbf{f}(\boldsymbol{X}) \in \mathbb{R}^{N \times C}$ follows a multivariate Gaussian distribution with mean zero and covariance matrix $\textbf{K}(\boldsymbol{X}, \boldsymbol{X}) \in \mathbb{R}^{NC \times NC}$ constructed from the set of kernel values $\left(\textbf{K}\left(x_i, x_j\right)\right)_{c,c'}$ with $1 \leq i, j \leq N$ and $1 \leq c, c' \leq C$, giving the form $\text{vec}(\textbf{f}(\boldsymbol{X}))\sim \mathcal{N}(\boldsymbol{0},\textbf{K}(\boldsymbol{X}, \boldsymbol{X}))$ where $\text{vec}(\cdot)$ denotes vectorization which transforms $\textbf{f}(\boldsymbol{X})$ from $\mathbb{R}^{N\times C}$ to $\mathbb{R}^{NC}$. 

\citet{alaa2017bayesian} used a multi-task GP for causal inference by setting the number of tasks equal to the number of potential outcomes, which is $C = 2$ for a binary treatment $T \in \{0, 1\}$. Here, $\textbf{f} = [f_0, f_1]^T$, where $f_0$ approximates the potential outcome for $T=0$, and $f_1$ approximates the potential outcome for $T=1$. Defining $\textbf{e} = [-1,1]^{T}$, we can then approximate CATE as 
\begin{equation}
\begin{split}
\text{CATE}(x) :=& \mathbb{E} \left[Y^{(1)} - Y^{(0)}|X\right] \\ =& \textbf{f}^{T}(x)\textbf{e} = f_1(x)-f_0(x).
\end{split} 
\end{equation}


\subsection{Coregionalization Models}
\label{sec:2.3}
A common approach to construct a multi-task GP is to use coregionalization models \cite{alvarez2012kernels}. For example, we can construct the matrix-valued kernel function $\textbf{K}$ from a single-output kernel by using the \emph{Intrinsic Coregionalization Model} (ICM \cite{goovaerts1997geostatistics}),
\begin{equation}
\textbf{K}_{\text{ICM}}(x,x') = k(x,x') \textbf{B},
\end{equation}
where $k: \mathcal{X} \times \mathcal{X} \rightarrow \mathbb{R}$ is a scalar-valued kernel function and $\textbf{B} \in \mathbb{R}^{C \times C}$ is called a \emph{coregionalization matrix}. If a function follows $\textbf{f} \sim \mathcal{GP} \left( \boldsymbol{0}, \textbf{K}_{\text{ICM}} \right)$, then each of its entries $\textbf{f}_c$ can be expressed as a linear combination of functions sampled from a GP with zero mean and covariance function $k$. That is, for a coregionalization matrix with $\text{rank}(\textbf{B})=R$, we have $\textbf{f}_c(x) = \sum_{r=1}^R a_c^r u^r(x)$ where $u^r(x) \sim \mathcal{GP} (0, k)$ for all $r \in [1,R]$. 

We can construct a more generalized model by using a \emph{Linear Model of Coregionalization} (LMC \cite{journel1976mining,goovaerts1997geostatistics}) to define $\textbf{K}$ ,
\begin{equation}
\textbf{K}_{\text{LMC}}(x,x') = \sum_{q=1}^Q k_q(x,x') \textbf{B}_q,
\end{equation}
where $k_q: \mathcal{X} \times \mathcal{X} \rightarrow \mathbb{R}$ is again a scalar-valued kernel function and $\textbf{B}_q \in \mathbb{R}^{C \times C}$ for all $q \in [1,Q]$. It is straightforward to see that the LMC can be viewed as a mixture of $Q$ ICMs. Likewise, if a function follows $\textbf{f} \sim \mathcal{GP} \left( \boldsymbol{0}, \textbf{K}_{\text{LMC}} \right)$, then we have $\textbf{f}_c(x) = \sum_{q=1}^Q \sum_{r=1}^{R_q} a_{c,q}^r u_q^r(x)$ where $u_q^r(x) \sim \mathcal{GP}(0, k_q)$ for all $r = 1, ..., R_q$ with $\text{rank}(\textbf{B}_q)=R_q$ and all $q = 1, ..., Q$.

\subsection{Relationship between NNs and GPs}
\label{sec:2.4}
As stated by \citet{matthews2018gaussian}, a random deep NN with appropriate activation function will converge \emph{in distribution} to a GP. Specifically, let $f_{\text{NN}}: \mathcal{X} \rightarrow \mathbb{R}$ be a function implemented by a NN with zero-mean i.i.d. parameters and continuous activation function $\phi$ which satisfies the following linear envelope property:
\begin{equation}
|\phi(u)| \leq \beta + m|u| \quad \forall u \in \mathbb{R}, 
\label{eq:5}
\end{equation}
if there exist $\beta, m \geq 0$. This property is satisfied by many common nonlinearities (e.g., ReLU, softplus, tanh, etc.). \textcolor{black}{Functions that violate this property (e.g., exponential) will induce heavy-tail behavior in the post activation.} Under these conditions, $f_{\text{NN}}$ will converge \emph{in distribution} to a GP in the infinite width limit,
\begin{equation}
f_{\text{NN}} \xrightarrow[]{d} \mathcal{GP} \left( \boldsymbol{0}, k_{\text{NN}} \right),
\end{equation}
where $k_{\text{NN}}: \mathcal{X} \times \mathcal{X} \rightarrow \mathbb{R}$ is a NN-implied kernel function and can be numerically estimated in a recursive manner \cite{lee2017deep}. However, for a single NN without a Bayesian formulation, its prediction can only be viewed as a \emph{sample} corresponding to a GP prior. To enable a full GP posterior interpretation, we adopt the sample-then-optimize approach proposed by \citet{matthews2017sample} by constructing and training a deep ensemble as we will discuss in the following section.



\section{Causal Multi-task Deep Ensemble}
\label{sec:3}

\begin{figure}[t!]
\centering
\includegraphics[width=0.47\textwidth]{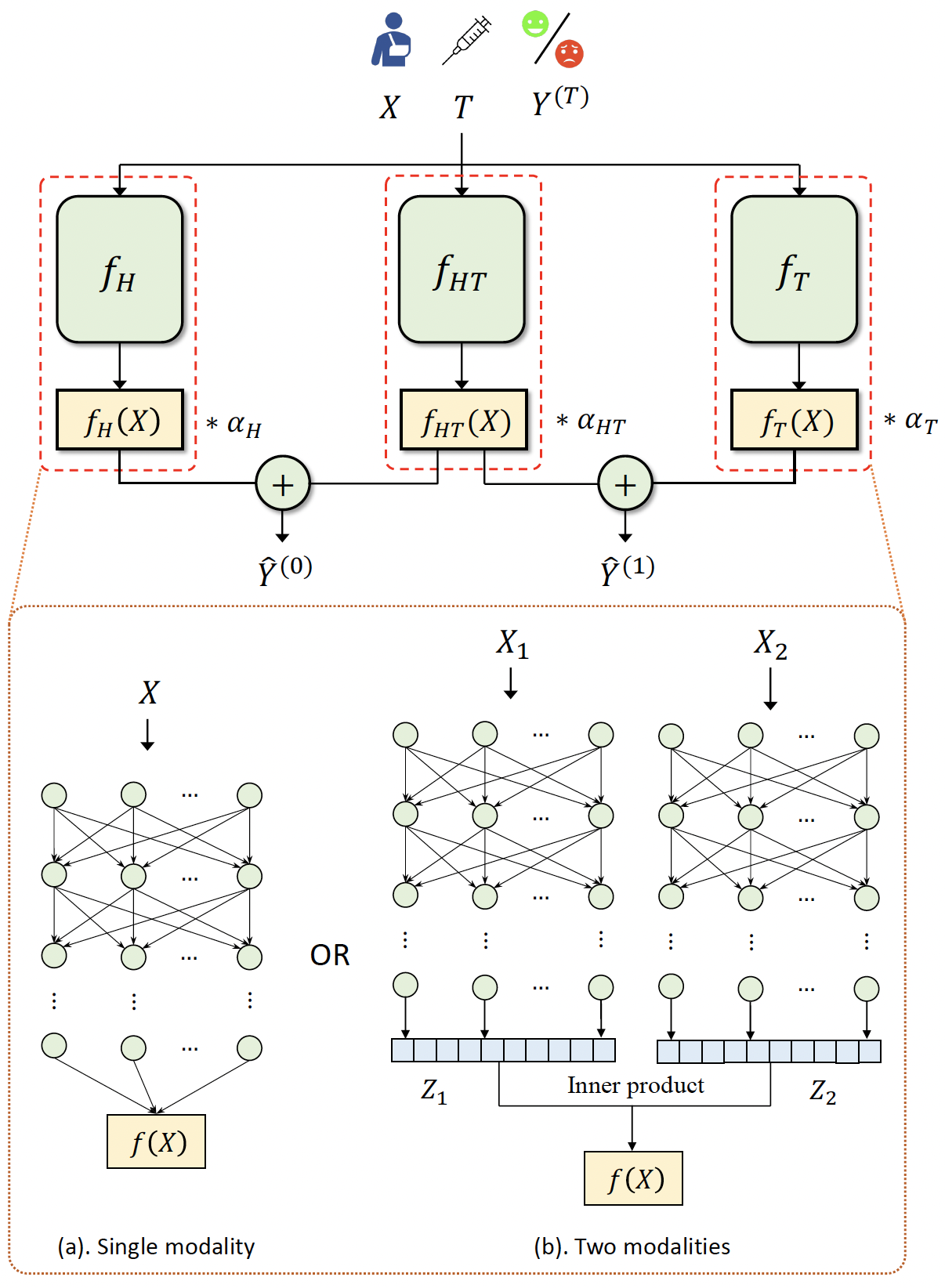}
\caption{The overall architecture of \textbf{a baselearner} in our causal multi-task deep ensemble where $f$ in (a) and (b) can be one of the $f_H, f_{HT},$ or $f_T$. Here the treatment assignment indicator $T$ is only used to obtain the corresponding factual outcome $Y^{(T)}$ for training and is not passed into $f_H$, $f_T$, or $f_{HT}$ as an input.}
\label{fig:1}
\vspace{-2.5mm}
\end{figure}

Here, we formally present our Causal Multi-task Deep Ensemble (CMDE) framework and elaborate on its relationship to coregionalization models. The ensemble's predictions are made by averaging over all its \emph{baselearners}. The architecture of a single baselearner in our ensemble is depicted in Figure \ref{fig:1} where features $X$ are passed into 3 neural networks $f_H, f_{T}, f_{HT}: \mathcal{X} \rightarrow \mathbb{R}$ separately as shown in Figure \ref{fig:1}a. Each baselearner follows the same architecture but has a different random initialization, which we will show corresponds to different draws from a multi-task GP prior. We expect $f_H$ and $f_T$ to learn group-specific information from control and treatment group, respectively, and $f_{HT}$ to learn shared information between the two groups. Each baselearner learns a multi-output function $\hat{\textbf{f}} = [\hat{f}_0, \hat{f}_1]^T$ which generates two outputs $\hat{Y}^{(0)}$ and $\hat{Y}^{(1)}$ representing the \emph{potential outcomes} for treatment assignment $T \in \{0,1\}$:
\begin{align}
&\hat{Y}^{(0)} = \hat{f}_0(X) \coloneqq \alpha_H f_H(X) + \alpha_{HT} f_{HT}(X) \label{eq:7}, \\
&\hat{Y}^{(1)} = \hat{f}_1(X) \coloneqq \alpha_{HT} f_{HT}(X) + \alpha_T f_T(X) \label{eq:8},
\end{align}
where $\alpha_H$, $\alpha_T$, and $\alpha_{HT}$ are trainable parameters that need to be initialized \emph{a priori}. With this formulation, we claim the following theorem.
\begin{theorem}
\vspace{2mm}
If all $f_H$, $f_T$, and $f_{HT}$ are neural networks with identical depth, zero-mean i.i.d. parameters with the same variance, and continuous activation function $\phi$ which satisfies the linear envelope property given in (\ref{eq:5}), then $\hat{\textbf{f}}$ converges in distribution to a GP with zero mean and ICM kernel in the infinite width limit \textbf{a priori}:
\begin{equation}
\hat{\textbf{f}} \xrightarrow[]{d} \mathcal{GP} \left( \boldsymbol{0}, \textbf{K}_{\text{ICM}} \right),
\end{equation}
where $\textbf{K}_{\text{ICM}}(x,x') = k_{\text{NN}}(x,x') \textbf{B}$, $x, x' \in \mathcal{X}$, and $k_{\text{NN}}$ is the kernel function implied by $f_H$, $f_T$, and $f_{HT}$, and
\begin{equation}
\textbf{B} = 
\begin{bmatrix}
\alpha_H^2 + \alpha_{HT}^2 & \alpha_{HT}^2 \\
\alpha_{HT}^2 & \alpha_T^2 + \alpha_{HT}^2
\end{bmatrix}.
\end{equation}
\label{thm:3.1}
\end{theorem}

\begin{proof}
To prove Theorem \ref{thm:3.1}, we calculate the covariance matrix between $\hat{\textbf{f}}(x)$ and $\hat{\textbf{f}}(x')$ as below (note that the expectations are taken with respect to the parameters of the functions inside the expectation):
{\allowdisplaybreaks
\begin{align}
&\text{cov} \left( \hat{\textbf{f}}(x), \hat{\textbf{f}}(x') \right) \\
&= \mathbb{E} \left[ \hat{\textbf{f}}(x)\hat{\textbf{f}}(x')^T \right] - \mathbb{E} \left[ \hat{\textbf{f}}(x) \right] \mathbb{E} \left[ \hat{\textbf{f}}(x') \right]^T \label{eq:12} \\
&= 
\begin{bmatrix}
\mathbb{E} [\hat{f}_0(x)\hat{f}_0(x')] & \mathbb{E} [\hat{f}_0(x)\hat{f}_1(x')] \\
\mathbb{E} [\hat{f}_1(x)\hat{f}_0(x')] & \mathbb{E} [\hat{f}_1(x)\hat{f}_1(x')]
\end{bmatrix}.
\label{eq:13}
\end{align}} 

From (\ref{eq:12}) to (\ref{eq:13}), we drop the term $\mathbb{E} \left[ \hat{\textbf{f}}(x) \right] \mathbb{E} \left[ \hat{\textbf{f}}(x') \right]^T$ as all the parameters in $f_H$, $f_T$, and $f_{HT}$ are initialized by i.i.d. zero mean random variables. We can then calculate each entry in (\ref{eq:13}) separately as follows:
\begin{align}
&\mathbb{E} [\hat{f}_0(x)\hat{f}_0(x')] \nonumber \\ 
&= \alpha_H^2 \mathbb{E}[f_H(x)f_H(x')] + \alpha_{HT}^2 \mathbb{E}[f_{HT}(x)f_{HT}(x')], \label{eq:14} \\
&\mathbb{E}[\hat{f}_0(x)\hat{f}_1(x')] 
= \alpha_{HT}^2 \mathbb{E}[f_{HT}(x)f_{HT}(x')], \label{eq:15} \\
&\mathbb{E} [\hat{f}_1(x)\hat{f}_0(x')] = \mathbb{E}[\hat{f}_0(x)\hat{f}_1(x')] \nonumber \\
&= \alpha_{HT}^2 \mathbb{E}[f_{HT}(x)f_{HT}(x')], \label{eq:16} \\
&\mathbb{E} [\hat{f}_1(x)\hat{f}_1(x')] \nonumber \\
&= \alpha_T^2 \mathbb{E}[f_T(x)f_T(x')] + \alpha_{HT}^2 \mathbb{E}[f_{HT}(x)f_{HT}(x')].
\label{eq:17}
\end{align}

From (\ref{eq:14}) to (\ref{eq:17}), we make use of the fact that the parameters of $f_H$, $f_T$, and $f_{HT}$ are independent of each other \textit{a priori}. Also, since $f_H$, $f_T$, and $f_{HT}$ share the same depth and initialization strategy, then as elaborated in Section \ref{sec:2.4}, we have $\mathbb{E}[f_H(x)f_H(x')] = \mathbb{E}[f_T(x)f_T(x')] = \mathbb{E}[f_{HT}(x)f_{HT}(x')] = k_{\text{NN}}(x,x')$ in the infinite width limit. Therefore, as the width of $f_H$, $f_T$, and $f_{HT}$ goes to infinity, we have:
\begin{align*}
&\mathbb{E} [\hat{f}_0(x)\hat{f}_0(x')] = (\alpha_H^2 + \alpha_{HT}^2) k_{\text{NN}}(x,x'), \\
&\mathbb{E}[\hat{f}_0(x)\hat{f}_1(x')] = \mathbb{E} [\hat{f}_1(x)\hat{f}_0(x')] = \alpha_{HT}^2 k_{\text{NN}}(x,x'), \\
&\mathbb{E} [\hat{f}_1(x)\hat{f}_1(x')] = (\alpha_T^2 + \alpha_{HT}^2) k_{\text{NN}}(x,x').
\end{align*}
By substituting the equations above back into (\ref{eq:13}), we get
\begin{equation}
\begin{split}
&\text{cov} \left( \hat{\textbf{f}}(x), \hat{\textbf{f}}(x') \right) = \\
&k_{\text{NN}}(x,x')
\begin{bmatrix}
\alpha_H^2 + \alpha_{HT}^2 & \alpha_{HT}^2 \\
\alpha_{HT}^2 & \alpha_T^2 + \alpha_{HT}^2
\end{bmatrix},
\end{split}
\end{equation}
which completes our proof for Theorem \ref{thm:3.1}.    
\end{proof}

In addition, by following a similar approach, we can also construct $\hat{\textbf{f}} = [\hat{f}_0, \hat{f}_1]^T$ as below:
\begin{align}
&\hat{f}_0(X) \coloneqq \sum_{q=1}^Q \alpha_H^q f_H^q(X) + \alpha_{HT}^q f_{HT}^q(X), \\
&\hat{f}_1(X) \coloneqq \sum_{q=1}^Q \alpha_{HT}^q f_{HT}^q(X) + \alpha_T^q f_T^q(X),
\end{align}
where $f_H^q$, $f_T^q$, and $f_{HT}^q$ share the same depth and initialization strategy for the same value of $q$. With this formulation, it can be shown that $\hat{\textbf{f}}$ converges \emph{in distribution to} a GP with zero mean and LMC kernel in the infinite width limit \emph{a priori} (see detailed proof in Appendix \ref{appx:A}):
\begin{equation}
\hat{\textbf{f}} \xrightarrow[]{d} \mathcal{GP} \left( \boldsymbol{0}, \textbf{K}_{\text{LMC}} \right),
\end{equation}
where $\textbf{K}_{\text{LMC}}(x,x') = \sum_{q=1}^Q k_{\text{NN}}^q(x,x') \textbf{B}_q$. Here $k_{\text{NN}}^q$ is the kernel function implied by $f_H^q$, $f_T^q$, or $f_{HT}^q$ and
\begin{equation}
\textbf{B}_q = 
\begin{bmatrix}
\left(\alpha_H^q\right)^2 + \left(\alpha_{HT}^q\right)^2 & \left(\alpha_{HT}^q\right)^2 \\
\left(\alpha_{HT}^q\right)^2 & \left(\alpha_T^q\right)^2 + \left(\alpha_{HT}^q\right)^2
\end{bmatrix}.
\end{equation}
In theory, our method can also be extended to the multiple-treatment case $T \in \{1,...,C\}$, where we can construct $\hat{\textbf{f}} = \left[ \hat{f}_{1}, ..., \hat{f}_{C} \right]^T$ as follows:
\begin{equation}
\begin{split}
\hat{f}_c(X) &\coloneqq \sum_{d=1}^{c-1} \alpha_{dc}f_{dc}(X) + \alpha_c f_c(X) + \sum_{d=c+1}^{C} \alpha_{cd}f_{cd}(X) \\ &\forall \; c = 1, ..., C,
\end{split}
\end{equation}
where $f_c$ learns the group-specific information and $f_{dc}, f_{cd}$ learn the shared information. With this formulation, we can also prove that $\hat{\textbf{f}}$ converges in distribution to a GP with an ICM kernel as elaborated in Appendix \ref{appx:B}. Our framework can also be simplified so each baselearner only contains 2 networks as shown in Appendix \ref{appx:C}. However, we will stick to the 3-network architecture in our experiments as it facilitates the explanation of the role of each network and enhances the clarity of our statement. \textcolor{black}{Note that we only state the equivalence between CMDE and a multi-task GP \emph{a priori}. According to \citet{he2020bayesian}, the equivalence between a deep ensemble and a GP may still hold \emph{a posteriori} (i.e., after training) if we augment the forward pass of each NN in the baselearner by adding a random and untrainable function $\delta(\cdot)$. However, we do not claim this equivalence, as the parameters of $f_H$, $f_T$, and $f_{HT}$ become dependent on each other. Consequently, we are unable to eliminate the cross-terms when calculating $\text{cov} \left( \hat{\textbf{f}}(x), \hat{\textbf{f}}(x') \right)$.}

\subsection{Extension to Multi-modal Covariates}
\label{sec:3.1}
In some cases, the covariates $X$ contain multiple modalities (e.g., $X = \left\{X_1, X_2\right\}$ where $X_1$ is an image and $X_2$ is in a tabular format). As illustrated in Figure \ref{fig:1}b, we adapt CMDE to such situations by introducing an inner product between the extracted representations from each modalities of $X$. Specifically, let $Z_m$ be the extracted representation from the $m^{th}$ input modality by a neural network. We can construct $f$ as follows
\begin{equation}
f(X) \coloneqq \sum_j \prod_m (Z_m)_j,
\end{equation}
where $(Z_m)_j$ represents the $j^{th}$ entry in $Z_m$ and $f$ is one of the $f_H$, $f_{HT}$, or $f_T$. As proved by Lee et al. \citeyearpar{lee2017deep} and Jiang et al. \citeyearpar{jiang2022incorporating}, this mechanism yields a multiplicative kernel $k_{\text{mul}} = \prod_m \left(k_{\text{NN}}\right)_m$ where $\left(k_{\text{NN}}\right)_m$ is the kernel function implied by the neural network used to extract the representation $Z_m$ from the input modality $X_m$. With this formulation, the multi-output function $\hat{\textbf{f}}$ still converges to an ICM or LMC kernel (depending on how we construct $\hat{\textbf{f}}$) except that we replace $k_{\text{NN}}$ with $k_{\text{mul}}$. 

\subsection{Training of CMDE}
\label{sec:3.2}
A common goal in causal inference is to minimize the \emph{precision in estimating heterogeneous effect} (PEHE \cite{hill2011bayesian}) loss, which is defined as
\begin{equation}
\begin{split}
&\mathcal{\hat{L}}\left(\hat{\textbf{f}}; \boldsymbol{Y}^{(\boldsymbol{T})}, \boldsymbol{Y}^{(\boldsymbol{1-T})}\right) = \\
&\frac{1}{N} \sum_{i=1}^{N}
\left(\hat{\textbf{f}}^{T}(x_{i})\textbf{e}-(1-2t_{i}) \left(y_{i}^{(1-t_{i})}-y_{i}^{(t_{i})} \right) \right)^{2}, 
\end{split}
\end{equation}
where $\textbf{e} = [-1,1]^{T}$, $\boldsymbol{Y}^{(\boldsymbol{T})} = \left\{y_{i}^{(t_{i})}\right\}_{i=1}^N$ are the factual outcomes, and $\boldsymbol{Y}^{(\boldsymbol{1-T})} = \{y_{i}^{(1-t_{i})}\}_{i=1}^N$ are the counterfactual outcomes.  To train CMDE, we want to minimize the following regularized empirical loss,
\begin{equation}
\hat{\textbf{f}}^{*} = \arg \,\!\min_{\hat{\textbf{f}} \in \mathcal{H_{\textbf{K}}}} \mathcal{\hat{L}}\left(\hat{\textbf{f}}; \boldsymbol{Y}^{(\boldsymbol{T})}, \boldsymbol{Y}^{(\boldsymbol{1-T})}\right) + \lambda \|\hat{\textbf{f}}\|^{2}_{\mathcal{H_{\textbf{K}}}},
\end{equation}
where $\mathcal{H_{\textbf{K}}}$ is a \emph{vector-valued Reproducing Kernel Hilbert Space} (vvRKHS) equipped with an inner product $\langle \cdot \; , \; \cdot \rangle_{\mathcal{H_{\textbf{K}}}}$, and reproducing kernel $\textbf{K}: \mathcal{X} \times \mathcal{X} \rightarrow \mathbb{R}^{2 \times 2}$. The regularization term smooths the loss based on the GP prior.

However, we cannot compute PEHE directly because we only observe the factual outcomes. Instead, we minimize the following empirical Bayesian PEHE risk with respect to $\hat{\textbf{f}}$ using stochastic gradient descent (SGD) with $L_2$ regularization:
\begin{equation}
\begin{split}
\hat{\mathcal{R}} &= \frac{1}{N} \sum_{i=1}^{N} \left( y_i^{\left(t_i\right)} - \mathbb{E} \left[ \hat{y}_i^{\left(t_i\right)}\big| x_i \right] \right)^2 \\
&+ \norm{\text{Var} \left[ \hat{y}_i^{\left(1-t_i\right)} \big| x_i \right]}_1, \label{eq:24}
\end{split}
\end{equation}
where $\hat{y}_i^{\left(t_i\right)}$ and $\hat{y}_i^{\left(1-t_i\right)}$ are the potential outcomes estimated by each estimator in CMDE as given in (\ref{eq:7}) and (\ref{eq:8}). The empirical mean and variance in (\ref{eq:24}) are computed over all the estimators in the ensemble.
It has been proved by \citet{alaa2017bayesian} that minimizing this risk is equivalent to minimizing the expectation of $\mathcal{\hat{L}}\left(\hat{\textbf{f}}; \boldsymbol{Y}^{(\boldsymbol{T})}, \boldsymbol{Y}^{(\boldsymbol{1-T})}\right)$ with respect to the posterior distribution of the counterfactual outcomes, which leads to a kernel that considers not only factual errors but also generalization to counterfactuals.

\section{Related Work}
\label{sec:4}
There exist several previous studies that focus on learning the potential outcomes with deep models or ensemble models, including Balancing Counterfactual Regression \cite{johansson2016learning}, the Counterfactual Regression Network (CFRNet \cite{shalit2017estimating}), and Bayesian Additive Regression Trees (BART) \cite{chipman2010bart}. These works generally attempt to learn a function $f: \mathcal{X} \times \{0,1\} \rightarrow \mathbb{R}$ which takes both the covariates and the treatment indicator as inputs. The treatment effect for an individual $x$ can thus be estimated as $\tau(x) \approx \hat{\tau}_f (x) = f(x,t=1) - f(x,t=0)$. However, the representation power of the treatment indicator $t$ can be significantly diluted when the dimension of the covariates $x$ becomes high, which can negatively affect the estimation of potential outcomes \cite{alaa2017bayesian,alaa2017deep}. Besides this line of work, \citet{alaa2017bayesian} proposed a multi-task GP framework that directly outputs the potential outcomes for both $t=0$ and $t=1$ by learning a multi-output function $\textbf{f}: \mathcal{X} \rightarrow \mathbb{R}^2$. The treatment effect in this case is estimated as $\tau(x) \approx \hat{\tau}_{\textbf{f}} (x) = \textbf{f}(x)^T \textbf{e}$ where $\textbf{e} = [-1, 1]^T$. Similar to CMGP, our deep ensemble model also learns a multi-output function while we replace the GPs with NNs to handle larger datasets and high-dimensional covariates, especially in cases where NNs outperform traditional GP approaches (e.g., images). 

\textcolor{black}{To the best of our knowledge, there is limited research on utilizing deep ensemble learning for causal inference in the existing literature. One notable study in this domain is the work by Hartford et al. \citeyearpar{hartford2021valid}, which focuses on treatment effect estimation using an ensemble of deep network-based instrumental variable estimators. Ensembles are commonly employed in machine learning and have been recognized for their effectiveness in reducing variance. In our approach, we adopt the ensemble method to approximate the full multi-task Gaussian process (GP) prior, whereby each baselearner can be viewed as a stochastic draw from this prior.}

It is also worth noting that CFRNet learns a \emph{balanced} representation such that the induced distributions for control and treatment groups look similar. Specifically, CFRNet consists of a network $\Phi$ which learns a representation from the covariates followed by two networks $h$ which learns hypotheses $h_0$ and $h_1$. This concept was employed in several subsequent studies on deep causal inference, including the Causal Effect Variational Autoencoder (CEVAE, \cite{louizos2017causal}), the Deep Counterfactual Network (DCN \cite{alaa2017deep}), and the Deep Orthogonal Networks (DONUT \cite{hatt2021estimating}). In contrast, our CMDE framework learns 3 representations which contains both group-specific and shared information from control and treatment groups. This provides more modeling flexibility, especially when the control and treatment groups are highly imbalanced.

As we explain in Section \ref{sec:3.1}, CMDE can be extended to handle multi-modal covariates. One recent study that focused on multi-modal causal inference is the Deep Multi-modal Structural Equations (DMSE) \cite{deshpandedeep}. A key distinction between DMSE and our method is that our causal graph does not include any latent variables. \textcolor{black}{While we acknowledge the potential benefits of latent-variable-based approaches in certain scenarios, their performance heavily relies on the complexity and accurate specification of the latent variable distribution, as emphasized by Rissanen and Marttinen \citeyearpar{rissanen2021critical}.}

The equivalence between NNs and GPs is also highly relevant to our method. Neal first proved that single-hidden-layer NNs become GPs as the width of the hidden layer goes to infinity \cite{neal1996priors,neal2012bayesian}. This proof was then extended to deep neural networks (DNNs) by Lee et al. \citeyearpar{lee2017deep} who designed an efficient implementation to calculate the NN-implied kernel and Matthews et al. \citeyearpar{matthews2018gaussian} who empirically evaluated the convergence rate via maximum mean discrepancy (MMD). Garriga-Alonso et al. \citeyearpar{garriga2018deep} also proved that convolutional neural networks (CNNs) are GPs in the limit of infinite number of channels. These findings allow us to simulate GP behavior using the outputs from NNs.

\begin{figure*}[t!]
\centering
\begin{subfigure}[t]{0.315\textwidth}
    \includegraphics[width=\linewidth]{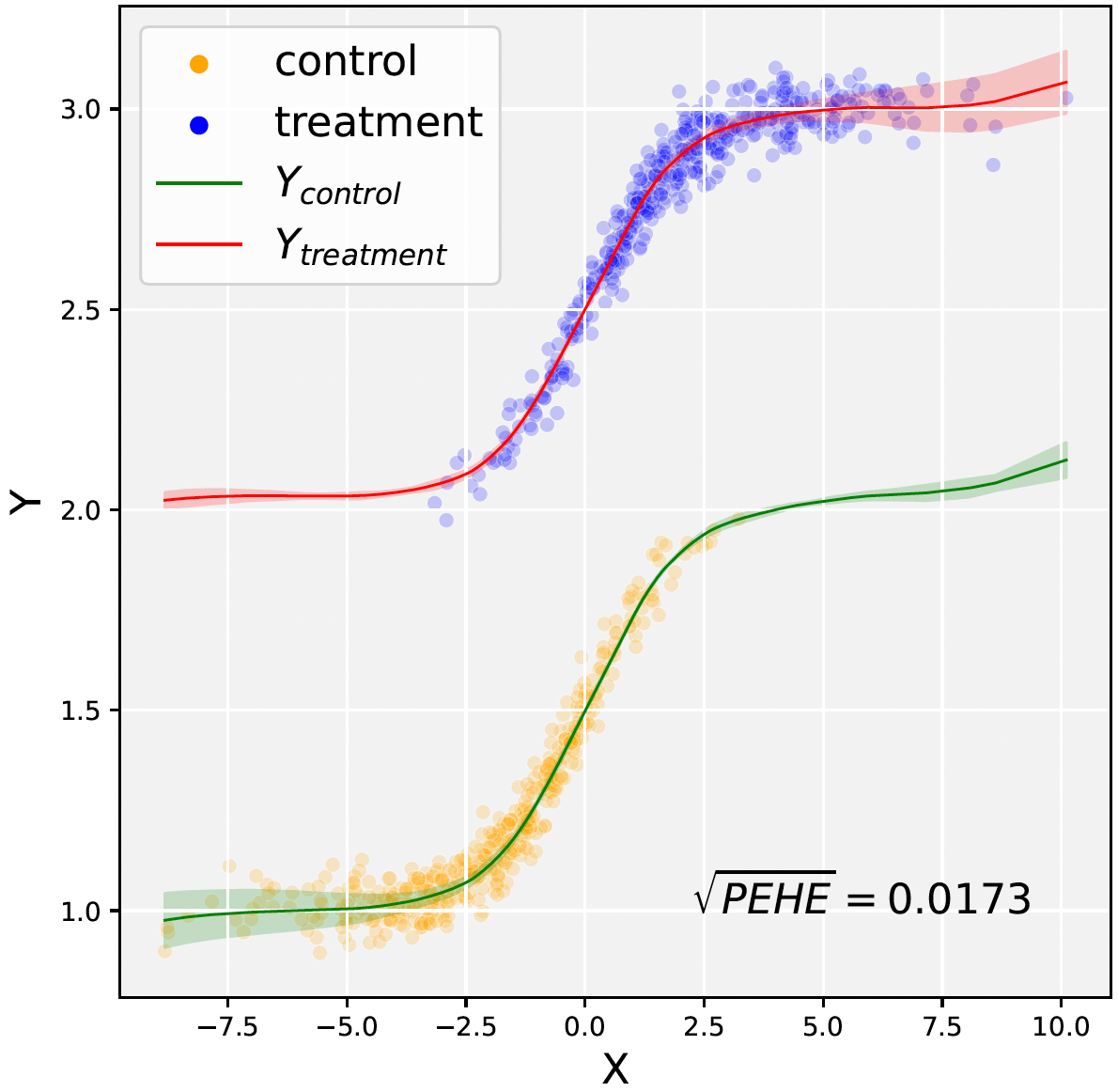}
    \label{fig:2a}
\end{subfigure}
\begin{subfigure}[t]{0.315\textwidth}
    \includegraphics[width=\linewidth]{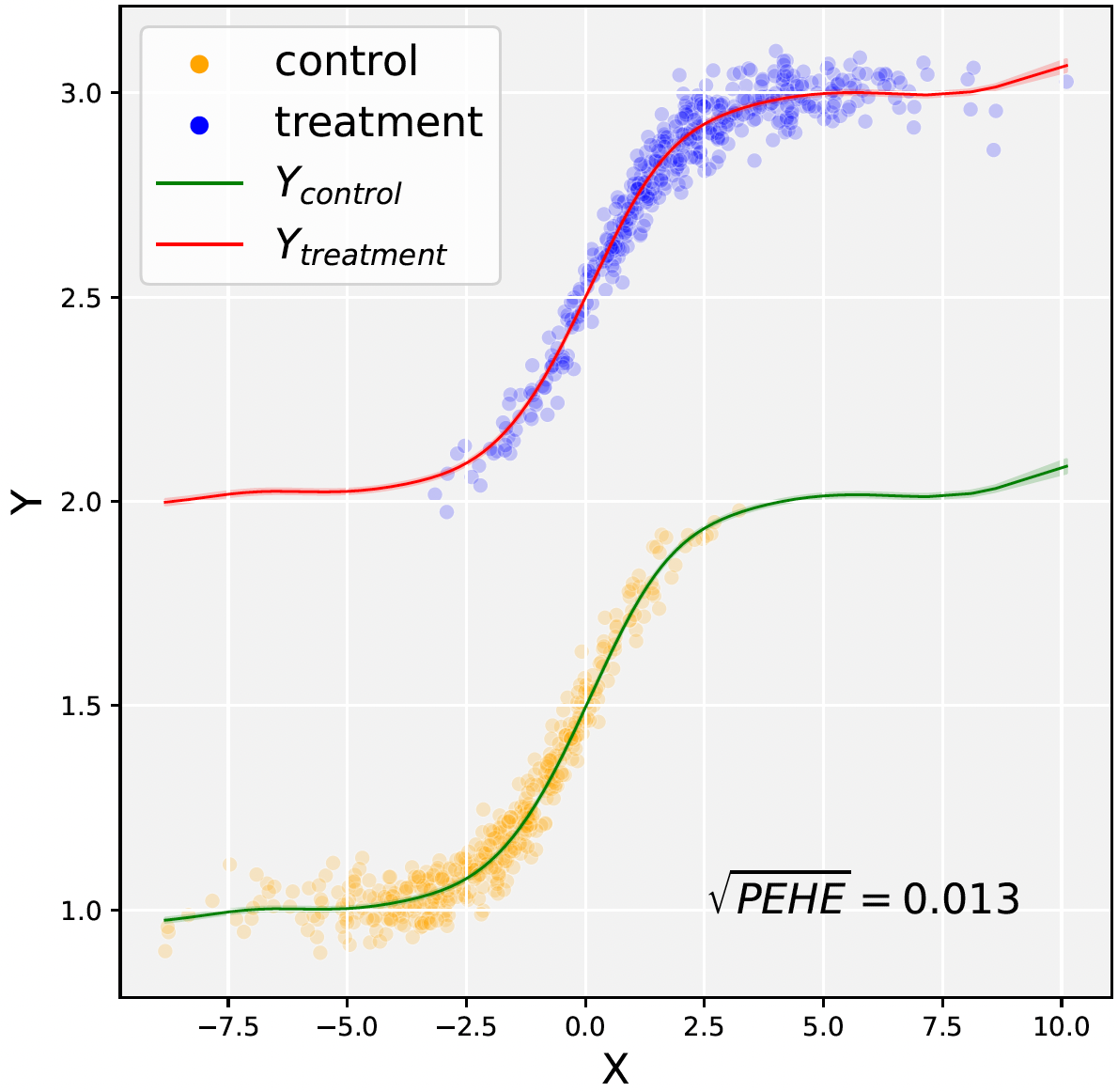}
    \label{fig:2b}
\end{subfigure}
\begin{subfigure}[t]{0.35\textwidth}
    \includegraphics[width=\linewidth]{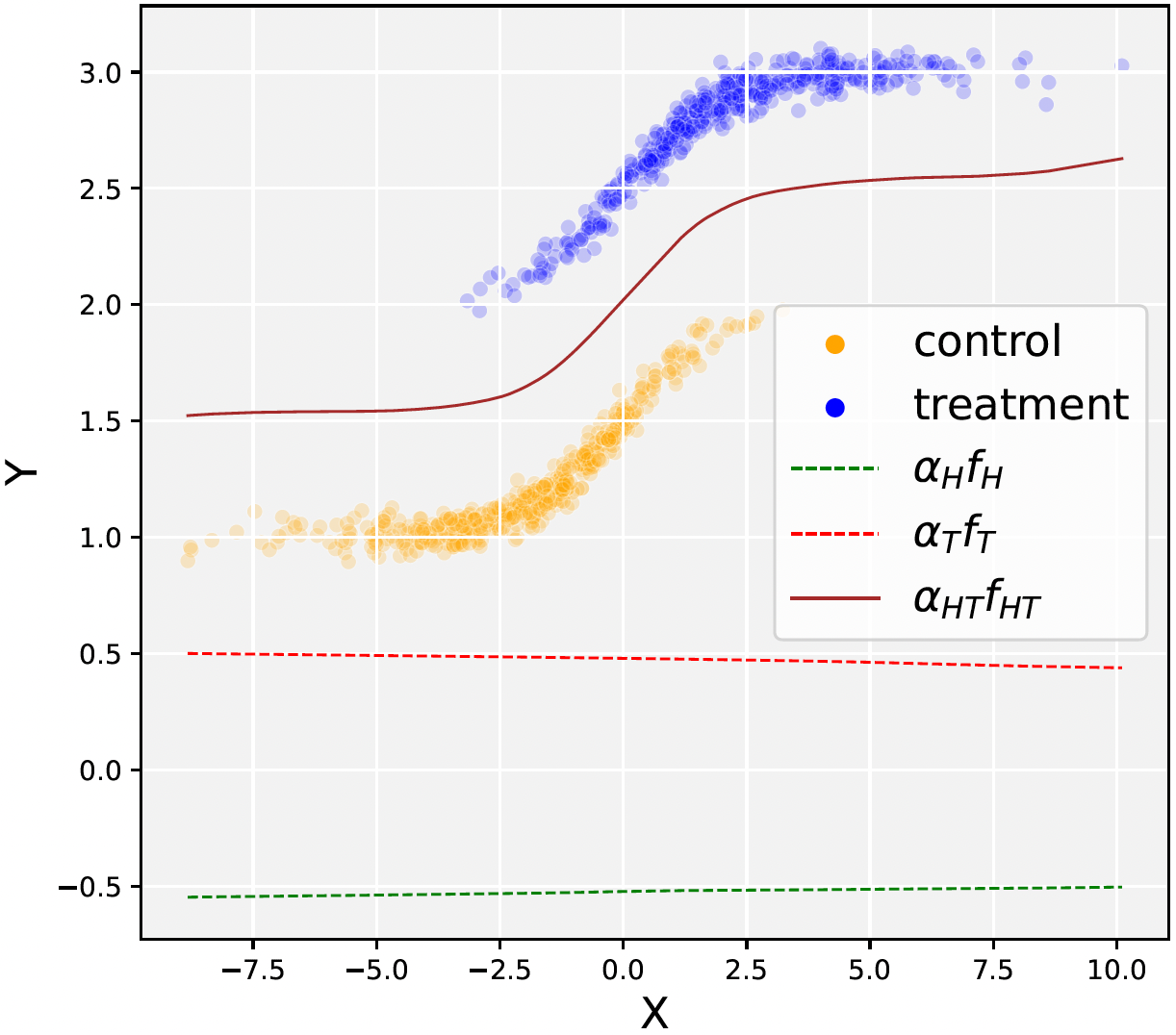}
    \label{fig:2c}
\end{subfigure}
\vspace{-3mm}
\caption{Predictions for the control and treatment groups on the synthetic dataset by CMDE (left) and multi-task GP with ICM kernel (middle) where dots represent observed samples, lines represent mean predictions, and shaded regions represent predicted values within 2 standard deviations. In addition, we also plot the contribution of group-specific and shared components for CMDE (right). It can be observed that $f_{HT}$ learns the shared features between the control and treatment groups and $f_H$ and $f_T$ learns the group-specific features.}
\label{fig:2}
\end{figure*}

\section{Experimental Results}
\label{sec:5}
We conduct experiments on a total of 6 datasets: one purely synthetic dataset, 3 benchmark datasets, and 2 datasets with multiple input modalities\footnotemark[1]. The detailed experimental setup is given in Appendix \ref{appx:D}.
\footnotetext[1]{The code to replicate all experiments is available at: \url{https://github.com/jzy95310/ICK/tree/main/experiments/causal_inference}}

\subsection{Synthetic Dataset}
\label{sec:5.1}
To empirically show the convergence of CMDE to its GP counterpart as the width of NNs goes to infinity, we first test CMDE on a synthetic dataset (see detailed data generation process in Appendix \ref{appx:D.1}) and compare it to a causal multi-task GP (CMGP) \cite{alaa2017bayesian} with an ICM kernel, 
\begin{align*}
&\textbf{K}_{\text{ICM}}(x,x') = k_{\text{NN}}(x,x') \textbf{B}, \\
&k_{\text{NN}}(x,x') = \textstyle \frac{2}{\pi} \sin^{-1} \left( \frac{2\tilde{x}^T \Sigma \tilde{x}'}{\sqrt{(1+2\tilde{x}^T \Sigma \tilde{x})(1+2\tilde{x}^{'^T} \Sigma \tilde{x}')}} \right),
\end{align*}
where $\tilde{x} = [1, x_1, x_2, ..., x_D]$ (e.g., a constant concatenated to the feature vector) and $\Sigma \in \mathbb{R}^{D \times D}$ is a pre-defined parameter representing the covariance of the weights in a single-hidden-layer NN. To evaluate how well we estimate the treatment effect, we use the PEHE metric 
\begin{equation}
\begin{split}
\textstyle
\epsilon_{\text{PEHE}} =
\frac{1}{N} \sum_{i=1}^N \left( \mathbb{E}_{y_i^{(0)}, y_i^{(1)} \sim \mathcal{Y}} \left[y_i^{(1)} - y_i^{(0)}\right]\right. \\-\left. \left(\hat{y}_i^{(1)} - \hat{y}_i^{(0)}\right) \right)^2,
\end{split}
\label{eq:25}
\end{equation}
where $y^{(0)}, y^{(1)}$ are true outcomes and $\hat{y}_i^{(0)}, \hat{y}_i^{(1)}$ are predicted outcomes. As shown in Figure \ref{fig:2}, the two methods yield very similar mean predictions and PEHE values except that CMDE tends to extrapolate with less confidence (i.e. higher standard deviation) where there exist fewer observed samples. We attribute this effect to only using 10 estimators for CMDE. We also plot the group-specific and shared components $\alpha_H f_H$, $\alpha_T f_T$, and $\alpha_{HT}f_{HT}$ in CMDE, which reveals that $f_{HT}$ learns the overall shape shared by the two response surfaces while $f_H$ and $f_T$ learn the magnitude of difference between the two surfaces.

\begin{figure}[t!]
\centering
\includegraphics[width=0.95\linewidth]{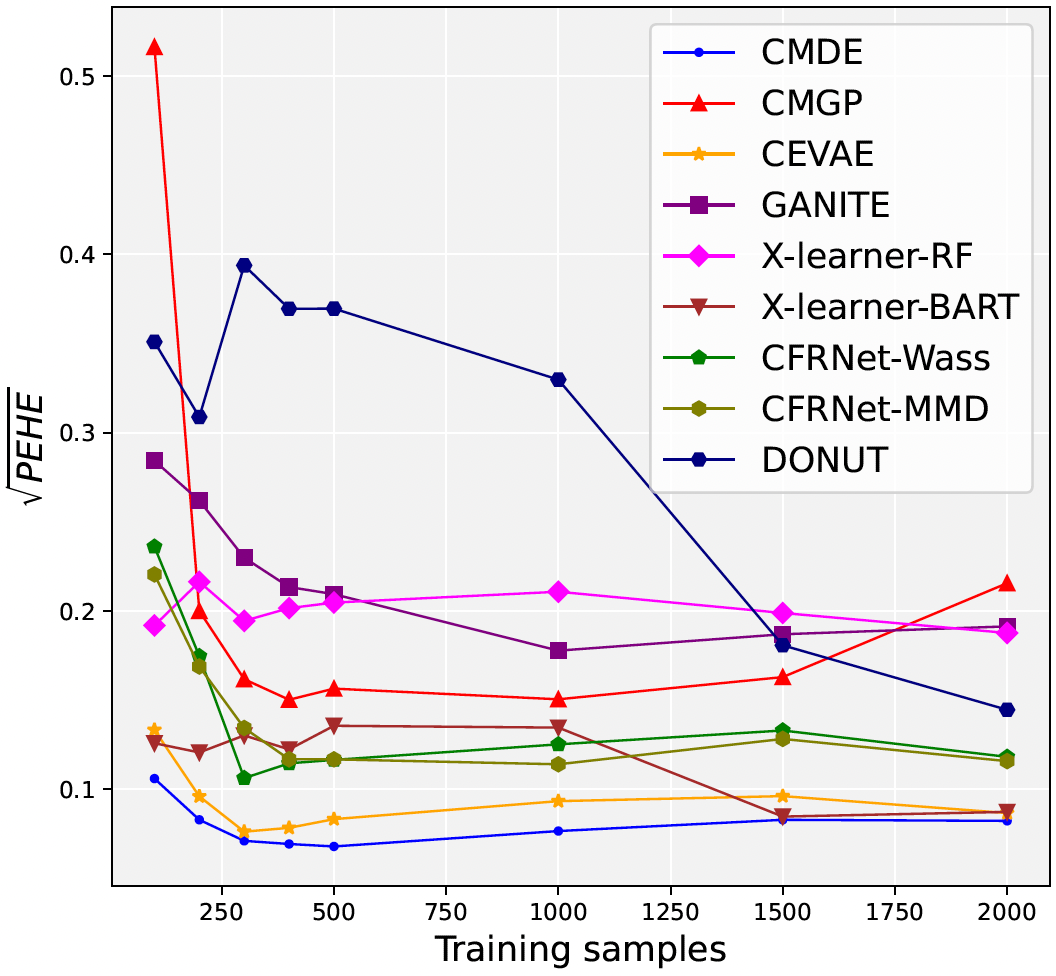}
\caption{Performance of CATE estimation $\left(\sqrt{\epsilon_{\text{PEHE}}}\right)$ on a dataset acquired from the Atlantic Causal Inference Conference held in 2019 (ACIC2019). Lower $\sqrt{\epsilon_{\text{PEHE}}}$ is better. Most of the methods exhibit saturated performance with $>500$ training samples.}
\label{fig:3}
\vspace{-3mm}
\end{figure}

\begin{table}[t!]
\setlength{\tabcolsep}{0.45em}
\centering
\begin{tabular}{l|cc|cc} 
\hline
~                                                        & \multicolumn{2}{c|}{Twins: $\sqrt{\hat{\epsilon}_{\text{PEHE}}}$}                                                                  & \multicolumn{2}{c}{Jobs: $\mathcal{R}_{\text{pol}}(\pi)$}                                                                    \\
~                                                        & \begin{tabular}[c]{@{}c@{}}In-\\sample\end{tabular} & \begin{tabular}[c]{@{}c@{}}Out-of-\\sample\end{tabular} & \begin{tabular}[c]{@{}c@{}}In-\\sample\end{tabular} & \begin{tabular}[c]{@{}c@{}}Out-of-\\sample\end{tabular}  \\ 
\hline
CMDE                                                     & .32 $\pm$ .00                                                  & \textbf{.32 $\pm$ .01}                                                   & \textbf{.05 $\pm$ .01}                                        & \textbf{.26 $\pm$ .02}                                          \\
CMGP                                                     & .44 $\pm$ .00                                                  & .44 $\pm$ .01                                                   & .12 $\pm$ .02                                        & .30 $\pm$ .02                                          \\
CEVAE                                                    & .32 $\pm$ .00                                                  & \textbf{.32 $\pm$ .01}                                                   & .11 $\pm$ .03                                        & .29 $\pm$ .03                                          \\
GANITE                                                   & .33 $\pm$ .00                                                  & .33 $\pm$ .01                                                   & .10 $\pm$ .02                                        & .30 $\pm$ .01                                          \\
X-RF   & \textbf{.30 $\pm$ .00}                                                  & .33 $\pm$ .01                                                   & N/A                                                & N/A                                                  \\
X-BART & .32 $\pm$ .00                                                  & \textbf{.32 $\pm$ .01}                                                   & N/A                                                & N/A                                                  \\
\begin{tabular}[c]{@{}l@{}}CFR-\\Wass\end{tabular} & .32 $\pm$ .00                                                  & \textbf{.32 $\pm$ .01}                                                   & .09 $\pm$ .03                                        & .28 $\pm$ .02                                          \\
\begin{tabular}[c]{@{}l@{}}CFR-\\MMD\end{tabular} & .32 $\pm$ .00                                                  & \textbf{.32 $\pm$ .01}                                                   & .08 $\pm$ .04                                        & .28 $\pm$ .03                                          \\
DONUT                                                    & .32 $\pm$ .00                                                  & \textbf{.32 $\pm$ .01}                                                   & .09 $\pm$ .05                                        & .27 $\pm$ .03                                          \\
\hline
\end{tabular}
\caption{Performance of CATE estimation on the Twins (left) and the Jobs (right) datasets for both in-sample and out-of-sample settings. Lower $\sqrt{\hat{\epsilon}_{\text{PEHE}}}$ or $\mathcal{R}_{\text{pol}}(\pi)$ is better. For Jobs, we do not report the results of X-learner as it directly estimates the individual treatment effect (ITE) instead of $y^{(t)}$.}
\label{tab:1}
\end{table}

\begin{figure*}[t!]
\begin{minipage}[c]{0.67\textwidth}
    \includegraphics[width=\textwidth]{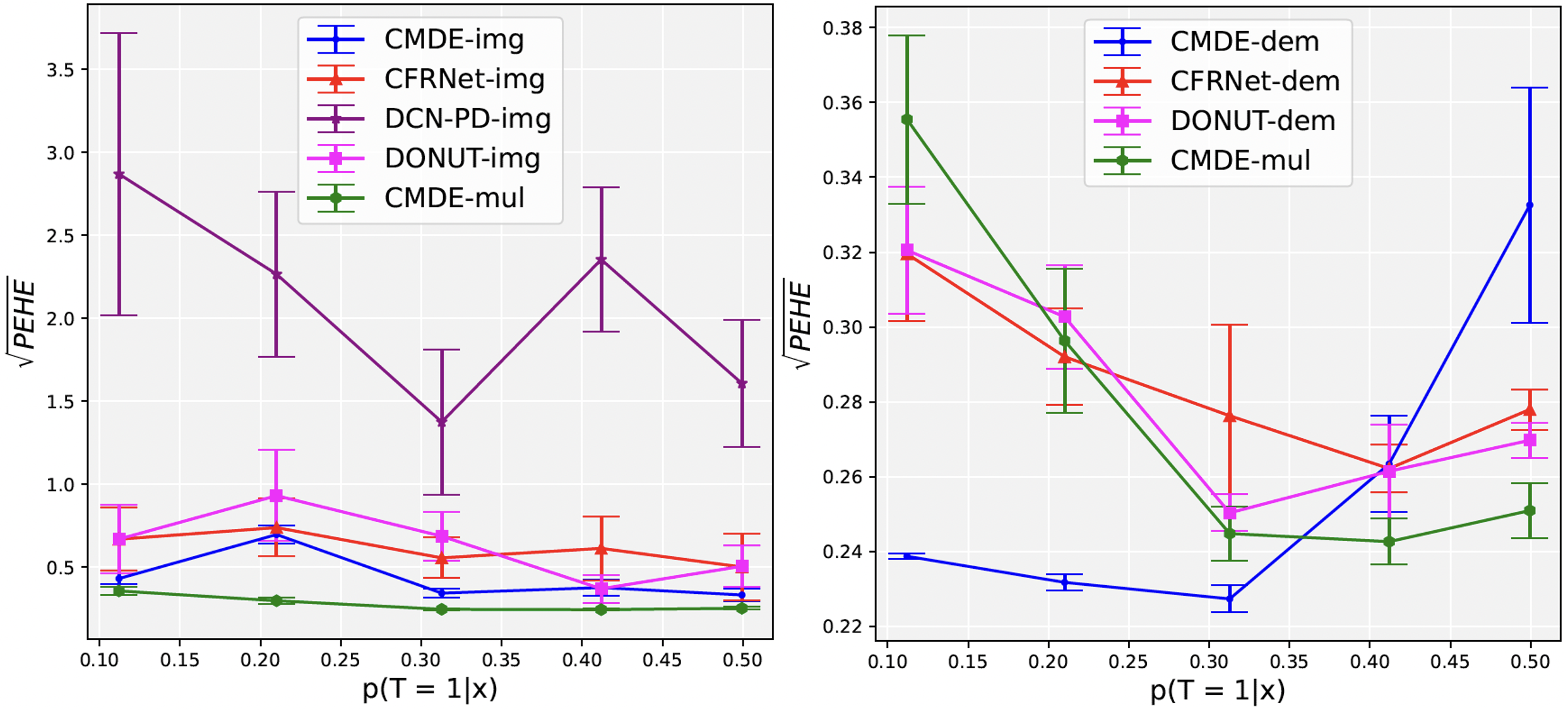}
\end{minipage}\hfill
\begin{minipage}[c]{0.30\textwidth}
    \caption{Results of CATE estimation $\left(\sqrt{\epsilon_{\text{PEHE}}}\right)$ on the semi-synthetic COVID-19 dataset with different propensities where the covariates are either X-ray images (left) or demographic information (right). CMDE with multi-modal covariates (both images and demographic information) are marked as CMDE-mul in both figures. The lines and error bars represent mean and half of the standard deviation of $\sqrt{\epsilon_{\text{PEHE}}}$, respectively. The error of DCN-PD with demographic information is too high so we do not show it in the right figure for better visualization.}
    \label{fig:4}
\end{minipage}
\end{figure*}

\subsection{Benchmark Datasets}
\label{sec:5.2}
We then evaluate CMDE on 3 frequently used benchmark datasets in the existing causal inference literature: a dataset acquired from the Atlantic Causal Inference Conference held in 2019 (ACIC2019 \cite{dorie2019automated}), the Twins dataset containing twins birth in the United States from 1989 to 1991 \cite{almond2005costs}, and the Jobs dataset studied by LaLonde \citeyearpar{lalonde1986evaluating} which is composed of randomized data based on state-supported work programs and non-randomized data from observational studies (see data pre-processing details in Appendix \ref{appx:D.2}). For evaluation metrics, we use $\epsilon_{\text{PEHE}}$ as given in (\ref{eq:25}) for ACIC2019 since we know the true expected values of $Y^{(0)}$ and $Y^{(1)}$. For the Twins dataset, since we observe both the factual and counterfactual outcomes (i.e. $y_i^{(0)}$ and $y_i^{(1)}$) on the paired data but do not know the underlying distribution $\mathcal{Y}$, we use the following empirical PEHE:
\begin{equation}
\textstyle \hat{\epsilon}_{\text{PEHE}} = \frac{1}{N} \sum_{i=1}^N \left( \left( y_i^{(1)}-y_i^{(0)} \right) - \left( \hat{y}_i^{(1)}-\hat{y}_i^{(0)} \right) \right)^2.
\end{equation}
For the Jobs dataset, only the factual outcomes are observed, so we use a metric called \emph{policy risk},
\begin{equation}
\textstyle \mathcal{R}_{\text{pol}}\left(\pi_f\right) = 1 - \frac{\sum_{i=1}^N y_i^{(t_i)} \mathbbm{1}[\pi_f(x_i)=t_i]}{\sum_{i=1}^N \mathbbm{1}[\pi_f(x_i)=t_i]},
\end{equation}
where we let the policy $\pi_f$ of a model $f$ to be $\pi_f(x) = 1$ if $f(x,t=1) > f(x,t=0)$ and $\pi_f(x) = 0$ otherwise. We compare CMDE with a total of 8 benchmark models: multi-task GP (CMGP \cite{alaa2017bayesian}), CEVAE \cite{louizos2017causal}, Generative Adversarial Nets (GANITE \cite{yoon2018ganite}), X-learner \cite{kunzel2019metalearners} of which the base learners are random forests (RF) and BART \cite{chipman2010bart}, Counterfactual Regression Network (CFRNet \cite{shalit2017estimating}) with 2-Wasserstein distance and Maximum Mean Discrepancy (MMD), and Deep Orthogonal Networks (DONUT \cite{hatt2021estimating}). For the ACIC2019 dataset, we vary the training set size to compare  algorithms in Figure \ref{fig:3}, and observe that CMDE gives the lowest error ($\epsilon_{\text{PEHE}}$) on the treatment effect. Furthermore, as shown in Table \ref{tab:1}, CMDE demonstrates competitive performance compared to other benchmark models in terms of PEHE on the Twins dataset, and outperforms all other benchmark models in terms of policy risk on Jobs.

\subsection{Datasets with Multi-modal Covariates}
\label{sec:5.3}
To demonstrate CMDE's strength in terms of handling high-dimensional and multi-modal covariates as described in Section \ref{sec:3.1}, we further adopt 2 datasets: a semi-synthetic COVID-19 dataset built upon a collection of patients' chest X-ray images and their corresponding demographic information and diagnosis (e.g., COVID-19 or other viral pneumonia, bacterial pneumonia, fungal pneumonia, etc.) \cite{cohen2020covid} and a real-world dataset from the Student-Teacher Achievement Ratio (STAR) experiment \cite{word1990student} with some features replaced by images with corresponding characteristics from the UTK dataset \cite{zhang2017age}. The details such as dataset pre-processing and model architectures can be found in Appendix \ref{appx:D.3}. 

We first conduct experiments on the semi-synthetic COVID-19 dataset under different propensity score settings. The results of CATE estimation for CMDE and benchmark deep causal models (i.e. CFRNet \cite{shalit2017estimating}, Deep Counterfactual Network with Propensity Dropout (DCN-PD \cite{alaa2017deep}), and DONUT \cite{hatt2021estimating}) are shown in Figure \ref{fig:4}. It can be observed that, with multi-modal covariates (i.e. both X-ray images and demographic information), CMDE-mul achieves the lowest PEHE compared to other benchmarks with only the X-ray images as covariates (or model inputs). In addition, CMDE-mul demonstrates superior performance compared to all other benchmarks with demographic information as covariates when the propensity score (i.e. $P(T = 1|x)$) is close to 0.5. However, when the propensity score ranges from 0.1 to 0.3, CMDE with only demographic information yields better results. \textcolor{black}{In other words, CMDE seems to benefit from using simpler covariate information when the control and treatment groups are relatively imbalanced. We attribute this phenomenon to a bias-variance tradeoff. Specifically, we note that the estimation variance tends to be larger when using simpler covariates compared to more complex ones. Conversely, incorporating additional information (e.g., bias reduction) through the use of more comprehensive covariates can lead to more accurate predictions. In situations where there is insufficient overlap between the control and treatment groups, the overall variance will increase. However, the increase in variance will be more substantial in cases involving complex covariates. Therefore, in scenarios with limited overlap, it may be advantageous to reduce the set of covariates, as the full estimation process may be dominated by error stemming from the variance term.}

We also compare CMDE with the same 3 benchmark models on another real-world dataset from the STAR experiment which studied the effect of class size on the students' performance and test scores. Since the original dataset corresponds to a randomized control trial and the true average treatment effect (ATE) can be estimated directly, we use the following ATE error as our evaluation metric in this experiment:
\begin{equation}
\epsilon_{\text{ATE}} = \bigg| \text{ATE}_{\text{true}} - \frac{1}{N} \sum_{i=1}^{N} \mathbb{E}\left[ \hat{y}_i^{(1)}|x_i \right] - \mathbb{E}\left[ \hat{y}_i^{(0)}|x_i \right] \bigg|.
\end{equation}
The results of CATE estimation are visualized as bar plots as shown in Figure \ref{fig:5}. We can see that CMDE outperforms other benchmark models with only the images or the students' information as covariates. Furthermore, with multi-modal covariates (i.e., both images and students' information), CMDE-mul yields the smallest ATE error with the lowest predictive uncertainty.

\begin{figure*}[t!]
\begin{minipage}[c]{0.61\textwidth}
    \includegraphics[width=\textwidth]{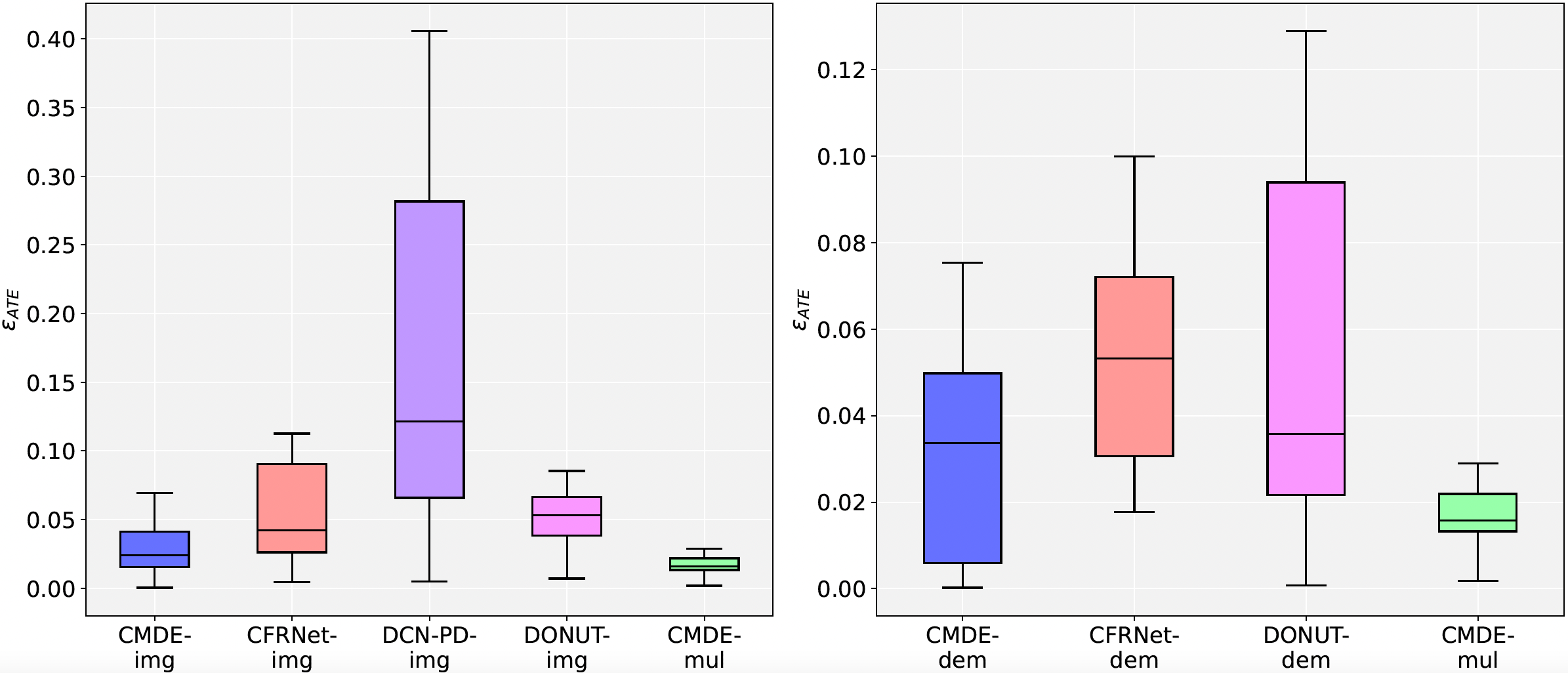}
\end{minipage}\hfill
\begin{minipage}[c]{0.35\textwidth}
    \caption{Box plots of CATE estimation $\left(\epsilon_{\text{ATE}}\right)$ on the STAR dataset where the covariates are either images (left) or the students' information (right). CMDE with multi-modal covariates (both images and students' information) are marked as CMDE-mul on the $x$-axis in both figures. The boxes extend from the $1^{\text{st}}$ quantile to the $3^{\text{rd}}$ quantile of $\epsilon_{\text{ATE}}$ with a line at the median. We do not show DCN-PD-dem on the right for cleaner visualization.}
    \label{fig:5}
\end{minipage}
\end{figure*}

\section{Discussion}
\label{sec:6}
\paragraph{NN architectures in CMDE} As we show in Section \ref{sec:3}, $f_H$, $f_T$, and $f_{HT}$ need to have the same depth and initialization strategy to guarantee CMDE's convergence to a multi-task GP with ICM kernel. While this requirement is essential for theoretical considerations, it does not pose a significant practical limitation. In practice, if the distributions of control and treatment groups exhibit significant difference, it is recommended either to use LMC kernel or to employ different NN architectures for each function.

\paragraph{Limitations}
While CMDE has shown excellent results in our experiments, we identify some potential limitations to be addressed in future work. For example, the hyperparameters (e.g., width, depth, initial parameter values, etc.) of NNs in CMDE can be hard to tune for specific tasks. We also find the performance of CMDE, in some cases, is sensitive to the initial values of the coefficients applied to NNs (e.g., $\alpha_H$, $\alpha_T$, and $\alpha_{HT}$), although we set these coefficients to be trainable. This requires us to have some prior knowledge about which type of information, group-specific or shared, is more dominant in specific datasets. \textcolor{black}{A detailed discussion is given in Appendix \ref{appx:E}.}

\paragraph{Applications}
We believe CMDE is applicable to a variety of real-world causal inference scenarios involving high-dimensional and multi-modal covariates, such as using medical records and images to estimate treatment effects in observational studies, or using A/B testing to determine the efficacy of a new version of user interface.

\paragraph{Societal Impact} Currently we are not aware of any new potential negative societal impacts of our work; however, like all machine learning methods that could be applied in the wild, the societal impact will depend on the task at hand. For example, the STAR dataset uses pictures of individuals to estimate causal effects; image processing can encode unwanted biases and checks should be in place before the deployment of any such system.

\section{Conclusion}
\label{sec:7}
We present a framework for estimating the causal effect of a treatment using a multi-task deep ensemble which learns both group-specific and shared information from control and treatment groups using separate neural networks. Theoretically, we demonstrate that our framework converges to a multi-task GP with an ICM/LMC kernel. We also provide empirical evidence of this relationship and visualize the contribution of each neural network components in our framework. Experimental results on various types of datasets demonstrate superior performance of CMDE compared to state-of-the-art approaches.

\section*{Acknowledgments}
Research reported in this publication was supported by the National Institute of Biomedical Imaging and Bioengineering of the National Institutes of Health and the the National Institute of Mental Health under Award Number R01EB026937. The content is solely the responsibility of the authors and does not necessarily represent the official views of the National Institutes of Health.

\bibliography{references}
\bibliographystyle{icml2023}

\newpage
\appendix
\onecolumn

\section{Proof of Convergence to LMC Kernel}
\renewcommand\thefigure{\thesection\arabic{figure}}
\setcounter{figure}{0}
\label{appx:A}
As stated in Section \ref{sec:3}, by constructing $\hat{\textbf{f}}$ as below:
\begin{align}
&\hat{f}_0(X) \coloneqq \sum_{q=1}^Q \alpha_H^q f_H^q(X) + \alpha_{HT}^q f_{HT}^q(X), \\
&\hat{f}_1(X) \coloneqq \sum_{q=1}^Q \alpha_{HT}^q f_{HT}^q(X) + \alpha_T^q f_T^q(X),
\end{align}
where $f_H^q$, $f_T^q$, and $f_{HT}^q$ share the same depth and initialization strategy for the same value of $q$, we can again separately compute each term in $\text{cov} \left( \hat{\textbf{f}}(x), \hat{\textbf{f}}(x') \right)$ as follows (again, note that the expectations are taken with respect to the parameters of the functions inside the expectation):
\begin{align}
\mathbb{E} [\hat{f}_0(x)\hat{f}_0(x')] &= \mathbb{E} \left[ \left( \sum_{q=1}^Q \alpha_H^q f_H^q(x) + \alpha_{HT}^q f_{HT}^q(x) \right) \left( \sum_{q=1}^Q \alpha_H^q f_H^q(x') + \alpha_{HT}^q f_{HT}^q(x') \right) \right] \nonumber \\
&= \sum_{q=1}^{Q} \left( \alpha_H^q \right)^2 \mathbb{E}\left[f_H^q(x) f_H^q(x')\right] + \left( \alpha_{HT}^q \right)^2 \mathbb{E}\left[f_{HT}^q(x) f_{HT}^q(x')\right], \label{eq:27} \\
\mathbb{E}[\hat{f}_0(x)\hat{f}_1(x')] &= \mathbb{E} \left[ \left( \sum_{q=1}^Q \alpha_H^q f_H^q(x) + \alpha_{HT}^q f_{HT}^q(x) \right) \left( \sum_{q=1}^Q \alpha_{HT}^q f_{HT}^q(x') + \alpha_T^q f_T^q(x') \right) \right] \nonumber \\
&= \sum_{q=1}^Q \left(\alpha_{HT}^q\right)^2 \mathbb{E} \left[ f_{HT}^q(x) f_{HT}^q(x') \right], \label{eq:28} \\
\mathbb{E} [\hat{f}_1(x)\hat{f}_0(x')] &= \mathbb{E} \left[ \left( \sum_{q=1}^Q \alpha_{HT}^q f_{HT}^q(x) + \alpha_T^q f_T^q(x) \right) \left( \sum_{q=1}^Q \alpha_H^q f_H^q(x') + \alpha_{HT}^q f_{HT}^q(x') \right) \right] \nonumber \\
&= \sum_{q=1}^Q \left(\alpha_{HT}^q\right)^2 \mathbb{E} \left[ f_{HT}^q(x) f_{HT}^q(x') \right], \label{eq:29} \\
\mathbb{E} [\hat{f}_1(x)\hat{f}_1(x')] &= \mathbb{E} \left[ \left( \sum_{q=1}^Q \alpha_{HT}^q f_{HT}^q(x) + \alpha_T^q f_T^q(x) \right) \left( \sum_{q=1}^Q \alpha_{HT}^q f_{HT}^q(x') + \alpha_T^q f_T^q(x') \right) \right] \nonumber \\
&= \sum_{q=1}^Q \left( \alpha_{HT}^q \right)^2 \mathbb{E}\left[f_{HT}^q(x) f_{HT}^q(x')\right] + \left( \alpha_T^q \right)^2 \mathbb{E} \left[ f_T^q(x) f_T^q(x') \right]. \label{eq:30}
\end{align}
For (\ref{eq:27}) to (\ref{eq:30}), we get rid of the cross terms based on the fact that the parameters of different neural networks all have zero mean and are independent of each other. Note that $f_H^q$, $f_T^q$, and $f_{HT}^q$ share the same depth and initialization strategy as stated in Theorem \ref{thm:3.1} for the same value of $q$, indicating that $\mathbb{E}[f_H^q(x)f_H^q(x')] = \mathbb{E}[f_T^q(x)f_T^q(x')] = \mathbb{E}[f_{HT}^q(x)f_{HT}^q(x')] = k_{\text{NN}}^q(x,x')$ for $q = 1, 2, ..., Q$ in the infinite width limit \emph{a priori}. Therefore, we have:
\begin{align*}
\mathbb{E} [\hat{f}_0(x)\hat{f}_0(x')] &= \sum_{q=1}^Q \left( \left( \alpha_H^q \right)^2 + \left( \alpha_{HT}^q \right)^2 \right) k_{\text{NN}}^q(x,x'), \\
\mathbb{E} [\hat{f}_0(x)\hat{f}_1(x')] &= \mathbb{E} [\hat{f}_1(x)\hat{f}_0(x')] = \sum_{q=1}^Q \left(\alpha_{HT}^q\right)^2 k_{\text{NN}}^q(x,x'), \\
\mathbb{E} [\hat{f}_1(x)\hat{f}_1(x')] &= \sum_{q=1}^Q \left( \left( \alpha_{HT}^q \right)^2 + \left( \alpha_T^q \right)^2 \right) k_{\text{NN}}^q(x,x').
\end{align*}
Substituting the expressions above back into $\text{cov} \left( \hat{\textbf{f}}(x), \hat{\textbf{f}}(x') \right)$ as given in (\ref{eq:13}), we get:
\begin{equation}
\text{cov} \left( \hat{\textbf{f}}(x), \hat{\textbf{f}}(x') \right) = \sum_{q=1}^Q k_{\text{NN}}^q(x,x') \textbf{B}_q \quad \text{where} \quad \textbf{B}_q = 
\begin{bmatrix}
\left( \alpha_H^q \right)^2 + \left( \alpha_{HT}^q \right)^2 & \left(\alpha_{HT}^q\right)^2 \\
\left(\alpha_{HT}^q\right)^2 & \left( \alpha_{HT}^q \right)^2 + \left( \alpha_T^q \right)^2
\end{bmatrix}.
\end{equation}
This proves that $\hat{\textbf{f}}$ will converge in distribution to a GP with zero mean and LMC kernel in the infinite width limit \emph{a priori}.

\section{Proof of Convergence to ICM Kernel for the Multiple-Treatment Case}
\renewcommand\thefigure{\thesection\arabic{figure}}
\setcounter{figure}{0}
\label{appx:B}
As elaborated in Section \ref{sec:3}, for multiple-treatment case $T \in \{1, ..., C\}$, we can construct $\hat{\textbf{f}} = \left[ \hat{f}_{1}, ..., \hat{f}_{C} \right]^T$ as follows:
\begin{equation}
\hat{f}_c(X) \coloneqq \sum_{d=1}^{c-1} \alpha_{dc}f_{dc}(X) + \alpha_c f_c(X) + \sum_{d=c+1}^{C} \alpha_{cd}f_{cd}(X) \; \forall \; c = 1, ..., C-1,
\end{equation}
where $f_c$ learns the group-specific information and $f_{dc}, f_{cd}$ learn the shared information. Similar to Appendix \ref{appx:A}, we can calculate each separate term in $\text{cov} \left( \hat{\textbf{f}}(x), \hat{\textbf{f}}(x') \right)$ (again, note that the expectations are taken with respect to the parameters
of the functions inside the expectation). For diagonal terms, we have:
\begin{align}
&\mathbb{E}[\hat{f}_c(x)\hat{f}_c(x')] \nonumber \\ 
&= \mathbb{E} \left[ \left( \sum_{d=1}^{c-1} \alpha_{dc}f_{dc}(x) + \alpha_c f_c(x) + \sum_{d=c+1}^{C} \alpha_{cd}f_{cd}(x) \right) \left( \sum_{d=1}^{c-1} \alpha_{dc}f_{dc}(x') + \alpha_c f_c(x') + \sum_{d=c+1}^{C} \alpha_{cd}f_{cd}(x') \right) \right] \nonumber \\
&= \sum_{d=1}^{c-1} \alpha_{dc}^2 \mathbb{E} \left[ f_{dc}(x)f_{dc}(x') \right] + \alpha_c^2 \mathbb{E}[f_c(x)f_c(x')] + \sum_{d=c+1}^{C} \alpha_{cd}^2 \mathbb{E} \left[ f_{cd}(x)f_{cd}(x') \right].
\end{align}
For off-diagonal terms, we have:
\begin{align}
&\mathbb{E} [\hat{f}_c(x)\hat{f}_{c'}(x')] \nonumber \\ 
&= \mathbb{E} \left[ \left( \sum_{d=1}^{c-1} \alpha_{dc}f_{dc}(x) + \alpha_c f_c(x) + \sum_{d=c+1}^{C} \alpha_{cd}f_{cd}(x) \right) \left( \sum_{d=1}^{c'-1} \alpha_{dc'}f_{dc'}(x') + \alpha_c' f_{c'}(x') + \sum_{d=c'+1}^{C} \alpha_{c'd}f_{c'd}(x') \right) \right] \nonumber \\
&= \alpha_{cc'}^2 \mathbb{E}[f_{cc'}(x)f_{cc'}(x')],
\end{align}
where $c < c'$. For $c > c'$, we have $\alpha_{cc'} = \alpha_{c'c}$ (i.e. the covariance matrix is symmetric). Here we again get rid of the cross terms based on the fact that the parameters of different neural networks all
have zero mean and are independent of each other. If all neural networks in this formulation (i.e. a total of $C(C+1)/2$ networks) share the same depth and initialization strategy as stated in Theorem \ref{thm:3.1}, indicating that $\mathbb{E}[f_c(x)f_c(x')] = k_{\text{NN}}(x,x') \; \forall \; c = 1, ..., C$ and $\mathbb{E}[f_{cc'}(x)f_{cc'}(x')] = k_{\text{NN}}(x,x') \; \forall \; c < c'$ and $c' = 1, ..., C$ in the infinite width limit \emph{a priori}, then we can further write the diagonal and off-diagonal terms in $\text{cov} \left( \hat{\textbf{f}}(x), \hat{\textbf{f}}(x') \right)$ as:
\begin{align*}
\mathbb{E}[\hat{f}_c(x)\hat{f}_c(x')] &= \left( \alpha_c^2 + \sum_{d=1}^{c-1} \alpha_{dc}^2 + \sum_{d=c+1}^{C} \alpha_{cd}^2 \right) k_{\text{NN}}(x,x'), \\
\mathbb{E} [\hat{f}_c(x)\hat{f}_{c'}(x')] &= \alpha_{cc'}^2 k_{\text{NN}}(x,x').
\end{align*}
With this, we derive:
\begin{equation}
\text{cov} \left( \hat{\textbf{f}}(x), \hat{\textbf{f}}(x') \right) = k_{\text{NN}}(x,x')
\begin{bmatrix}
\alpha_1^2 + \sum_{d=2}^{C} \alpha_{1d}^2 & \alpha_{12}^2 & \cdots & \alpha_{1C}^2 \\
\alpha_{12}^2 & \alpha_{12}^2 + \alpha_2^2 + \sum_{d=3}^C \alpha_{2d}^2 & \cdots & \alpha_{2C}^2 \\
\vdots & \vdots & \ddots & \vdots \\
\alpha_{1C}^2 & \alpha_{2C}^2 & \cdots & \sum_{d=1}^{C-1} \alpha_{dC}^2 + \alpha_C^2
\end{bmatrix}.
\end{equation}
This proves that $\hat{\textbf{f}}$ will converge in distribution to a GP with zero mean and ICM kernel in the infinite width limit \emph{a priori} when there exists a total of $C$ treatments, i.e. $T \in \{1, ..., C\}$. 

\section{Two-Network Architecture for CMDE}
\renewcommand\thefigure{\thesection\arabic{figure}}
\setcounter{figure}{0}
\label{appx:C}
The architecture of each baselearner in CMDE as presented in Section \ref{sec:3} can be simplified to the following two-network architecture (i.e. $f_A$ and $f_B$):
\begin{align}
&\hat{Y}^{(0)} = \hat{f}_0(X) \coloneqq \alpha_0 f_A(X) + \beta_0 f_B(X), \\
&\hat{Y}^{(1)} = \hat{f}_1(X) \coloneqq \alpha_1 f_A(X) + \beta_1 f_B(X).
\end{align}
Following a similar procedure as given in Appendices \ref{appx:A} and \ref{appx:B}, we have:
\begin{align}
\mathbb{E}[\hat{f}_0(x)\hat{f}_0(x')] &= \alpha_0^2 \mathbb{E}[f_A(x) f_A(x')] + \beta_0^2 \mathbb{E}[f_B(x) f_B(x')] = (\alpha_0^2 + \beta_0^2)k_{\text{NN}}(x, x'), \\
\mathbb{E}[\hat{f}_0(x)\hat{f}_1(x')] & = \alpha_0 \alpha_1 \mathbb{E}[f_A(x) f_A(x')] + \beta_0 \beta_1 \mathbb{E}[f_B(x)f_B(x')] = (\alpha_0 \alpha_1 + \beta_0 \beta_1) k_{\text{NN}}(x, x'), \\
\mathbb{E}[\hat{f}_1(x)\hat{f}_0(x')] &= \mathbb{E}[\hat{f}_0(x)\hat{f}_1(x')] = (\alpha_0 \alpha_1 + \beta_0 \beta_1) k_{\text{NN}}(x, x'), \\
\mathbb{E}[\hat{f}_1(x)\hat{f}_1(x')] &= \alpha_1^2 \mathbb{E}[f_A(x) f_A(x')] + \beta_1^2 \mathbb{E}[f_B(x) f_B(x')] = (\alpha_1^2 + \beta_1^2) k_{\text{NN}}(x, x').
\end{align}
The covariance function then becomes:
\begin{equation}
\text{cov} \left( \hat{\textbf{f}}(x), \hat{\textbf{f}}(x') \right) = k_{\text{NN}}(x, x') 
\begin{bmatrix}
\alpha_0^2 + \beta_0^2 & \alpha_0 \alpha_1 + \beta_0 \beta_1 \\
\alpha_0 \alpha_1 + \beta_0 \beta_1 & \alpha_1^2 + \beta_1^2
\end{bmatrix}.
\end{equation}
Therefore, $\hat{\textbf{f}}$ will still converge in distribution to a GP with zero mean and ICM kernel in the infinite width limit \emph{a priori} with this two-network architecture.

\section{Details of Experimental Setup}
\renewcommand\thefigure{\thesection\arabic{figure}}
\setcounter{figure}{0}
\label{appx:D}

\subsection{Synthetic Dataset}
\label{appx:D.1}
We construct the synthetic dataset in Section \ref{sec:5.1} by following the steps below. For $i = 1, 2, ..., N$, do
\begin{align*}
x_i &\sim \mathcal{N} \left( 0, \sigma_x^2 \right), \\
t_i &\sim \text{Bern}(p_i) \quad \text{where} \quad p_i = \frac{1}{1+\exp(-x_i)}, \\
\xi_i &\sim \mathcal{N} \left( 0, \sigma_{\xi}^2 \right), \\
\mu_i^{(0)} &= 1 + \frac{1}{1 + \exp(-x_i)}, \\
\mu_i^{(1)} &= 2 + \frac{1}{1 + \exp(-x_i)}, \\
y_i^{(0)} &= \mu_i^{(0)} + \xi_i, \\
y_i^{(1)} &= \mu_i^{(1)} + \xi_i, \\
y_i &= y_i^{(0)} \quad \text{if} \quad t_i = 0 \quad \text{else} \quad y_i^{(1)}.
\end{align*}
For our experiment, we set $\sigma_x^2 = 9$ and $\sigma_{\xi}^2 = 0.0025$ and sample $N = 3000$ data points. The CMDE model consists of 10 estimators where $f_H$, $f_T$, and $f_{HT}$ in each estimator are single-hidden-layer neural networks with ReLU activation and 2048 units in the hidden layer. We set the initial values of $\alpha_H$, $\alpha_T$, and $\alpha_{HT}$ to be $\alpha_H = 0$, $\alpha_T = 0$, and $\alpha_{HT} = 1$. All weight and bias parameters in $f_H$, $f_T$, and $f_{HT}$ are independently drawn from a normal distribution $\mathcal{N}(0, \sigma_w^2 I)$ \emph{a priori} and $\sigma_w^2 = 0.1$. 

\subsection{Benchmark Datasets}
\label{appx:D.2}
We elaborate the data pre-processing details in the sub-sections below. The model hyperparameter details are listed in Table \ref{tab:2}. Also, note that for Twins and Jobs dataset, we use both the training and validation set to evaluate the models for in-sample setting and just the test set to evaluate the models for out-of-sample setting. We repeat the experiments on Twins and Jobs dataset 10 times and report the mean and the standard deviation as given in Table \ref{tab:1}.

\subsubsection{Atlantic Causal Inference Conference (ACIC) Dataset}
\label{appx:D.2.1}
The covariates in ACIC2019 dataset are either simulated or drawn from publicly available datasets. We take the high-dimensional version (where we have 185 covariates in total) for our experiments. The full training and test sets contain a total of 6.4M and 16K data points, respectively. Due to time and memory constraints, we pick a small subset containing 2000 data points from each of the training and test sets. The download link is provided below: \\
\url{https://sites.google.com/view/acic2019datachallenge/data-challenge?pli=1}

\subsubsection{Twins Dataset}
\label{appx:D.2.2}
The Twins dataset contains the information of twin births in the United States from 1989 to 1991. It contains 40 covariates pertaining to pregnancy, twin births, and parents. The treatment is defined as $T = 1$ as being the heavier twin and $T = 0$ as being the lighter twin. The outcome is defined as the 1-year mortality. The full dataset contains a total of 11400 data points and we average over 10 train-validation-test splits with a ratio of 56:24:20.

\subsubsection{Jobs Dataset}
\label{appx:D.2.3}
The Jobs dataset studied by LaLonde is a widely used benchmark where the treatment $T$ is job training and the outcome $Y$ is the individual's income in 1975. The covariates include 8 variables such as age, education, race, and income in 1974. The dataset consists of a randomized portion based on the National Supported Work program (722 samples) and a non-randomized portion acquired from observational studies (2490 samples). Before conducting the experiment, we convert $Y$ (income in 1975) into binary outcomes (i.e. employed/unemployed or $\mathbbm{1}[Y = 0]$). The test set is sampled \emph{only from the randomized portion} and we average over 10 train-validation-test splits with a ratio of 56:24:20.

\begin{table}[t!]
\setlength{\tabcolsep}{0.9em}
\centering
\begin{tabular}{l|c|c|c} 
\hline
~                                                        & ACIC                                                                                                 & Twins & Jobs  \\ 
\hline
CMDE                                                     & \begin{tabular}[c]{@{}c@{}}number of estimators $= 10$\\LMC kernel ($Q = 2$),\\depth $= 2$, width $= 512$, \\$\alpha_H^1 = \alpha_T^1 = \alpha_{HT}^1 = 1$,\\$\alpha_H^2 = \alpha_T^2 = \alpha_{HT}^2 = 1$, \\softplus activation ~\end{tabular} & \begin{tabular}[c]{@{}c@{}}number of estimators $= 10$\\LMC kernel ($Q = 2$),\\depth $= 2$, width $= 512$, \\$\alpha_H^1 = \alpha_T^1 = 1, \alpha_{HT}^1 = 0.1$,\\$\alpha_H^2 = \alpha_T^2 = 1, \alpha_{HT}^2 = 0.1$, \\tanh activation ~\end{tabular}     & \begin{tabular}[c]{@{}c@{}}number of estimators $= 10$\\LMC kernel ($Q = 2$),\\depth $= 2$, width $= 512$, \\$\alpha_H^1 = \alpha_T^1 = 1, \alpha_{HT}^1 = 0.1$,\\$\alpha_H^2 = \alpha_T^2 = 1, \alpha_{HT}^2 = 0.1$, \\tanh activation ~\end{tabular}     \\ 
\hline
CMGP                                                     & \begin{tabular}[c]{@{}c@{}}LMC kernel with \\RBF base kernel~\end{tabular}                                                                                                   & N/A     & \begin{tabular}[c]{@{}c@{}}LMC kernel with \\RBF base kernel~\end{tabular}     \\ 
\hline
CEVAE                                                    & $\dagger$                                                                                                    & $\dagger$     & $\dagger$     \\ 
\hline
GANITE                                                   & \begin{tabular}[c]{@{}c@{}}$k_G = k_I = 256$, \\depth $= 0$, $h_{dim} = 100$, \\$\alpha = 0.1$, $\beta = 0$ ~\end{tabular}                                                                                                    & \begin{tabular}[c]{@{}c@{}}$k_G = k_I = 128$, \\depth $= 5$, $h_{dim} = 8$, \\$\alpha = 2$, $\beta = 2$ ~\end{tabular}     & \begin{tabular}[c]{@{}c@{}}$k_G = k_I = 128$, \\depth $= 3$, $h_{dim} = 4$, \\$\alpha = 1$, $\beta = 5$ ~\end{tabular}     \\ 
\hline
\begin{tabular}[c]{@{}l@{}}X-learner-\\RF\end{tabular}   & number of estimators = $100$                                                                                                    & number of estimators = $100$     & N/A     \\ 
\hline
\begin{tabular}[c]{@{}l@{}}X-learner-\\BART\end{tabular} & number of estimators = $100$                                                                                                    & number of estimators = $100$     & N/A     \\ 
\hline
\begin{tabular}[c]{@{}l@{}}CFRNet-\\Wass\end{tabular}    & \begin{tabular}[c]{@{}c@{}}depth ($\phi$ and $h$) $= 2$, \\width ($\phi$ and $h$) $= 512$, \\$\alpha = 0.1$, ReLU activation ~\end{tabular}                                                                                                    & \begin{tabular}[c]{@{}c@{}}depth ($\phi$ and $h$) $= 2$, \\width ($\phi$ and $h$) $= 512$, \\$\alpha = 1$, tanh activation ~\end{tabular}     & \begin{tabular}[c]{@{}c@{}}depth ($\phi$ and $h$) $= 2$, \\width ($\phi$ and $h$) $= 512$, \\$\alpha = 1$, tanh activation ~\end{tabular}     \\ 
\hline
\begin{tabular}[c]{@{}l@{}}CFRNet-\\MMD\end{tabular}     & \begin{tabular}[c]{@{}c@{}}depth ($\phi$ and $h$) $= 2$, \\width ($\phi$ and $h$) $= 512$, \\$\alpha = 0.1$, ReLU activation ~\end{tabular}                                                                                                    & \begin{tabular}[c]{@{}c@{}}depth ($\phi$ and $h$) $= 2$, \\width ($\phi$ and $h$) $= 512$, \\$\alpha = 1$, tanh activation ~\end{tabular}     & \begin{tabular}[c]{@{}c@{}}depth ($\phi$ and $h$) $= 2$, \\width ($\phi$ and $h$) $= 512$, \\$\alpha = 1$, tanh activation ~\end{tabular}     \\ 
\hline
DONUT                                                    & \begin{tabular}[c]{@{}c@{}}depth ($\phi$ and $h$) $= 2$, \\width ($\phi$ and $h$) $= 512$, \\ReLU activation ~\end{tabular}                                                                                                    & \begin{tabular}[c]{@{}c@{}}depth ($\phi$ and $h$) $= 2$, \\width ($\phi$ and $h$) $= 512$, \\tanh activation ~\end{tabular}     & \begin{tabular}[c]{@{}c@{}}depth ($\phi$ and $h$) $= 2$, \\width ($\phi$ and $h$) $= 512$, \\tanh activation ~\end{tabular}     \\
\hline
\end{tabular}
\caption{Model hyperparameters used for CMDE and other benchmark models in the ACIC2019, Twins, and Jobs experiments. $\dagger$ To save space, for CEVAE, please refer to the source code for hyperparameter details.}
\label{tab:2}
\end{table}

\subsection{Datasets with Multi-modal Covariates}
\label{appx:D.3}
We give the details of data generation and pre-processing for each experiment in the sub-sections below. For CMDE and benchmark deep causal models, we use a convolutional neural network (CNN) architecture when we have images as covariates and a fully connected neural network architecture when we have tabular data (e.g., demographic information in COVID-19 dataset) as covariates. The details of model architectures are displayed in Table \ref{tab:3}. We repeat the experiments on both datasets 10 times and report the corresponding statistics as given in Figures \ref{fig:4} and \ref{fig:5}.

\subsubsection{Data Generation Procedure of the Semi-synthetic COVID-19 Dataset}
\label{appx:D.3.1}
We create a semi-synthetic dataset based on a publicly available COVID-19 X-ray dataset. The dataset includes 951 images and some other demographic and image-related information, collected from several public sources \cite{cohen2020covid}. After data cleaning and imputation, 857 samples are used in our analysis and we use a train-validation-test splito ratio of 40:20:40. The variables we include are patient id, offset, sex, age, RT-PCR-positive, survival, intubated, intubation-present, went-icu, in-icu.

We generate potential outcomes using both demographic and image information. For image, we categorize the diagnosis of the X-ray into the following categories: viral pneumonia, bacterial pneumonia, fungal pneumonia, pneumonia caused by other causes (lipoid and aspiration), pneumonia by unknown cause, and tuberculosis. We use these categories to represent image information. The potential outcome $Y$ is defined as general overall severity of diseases and $T$ is the binary treatment. Larger value of $Y$ indicates worse prognosis. The potential outcomes are generated by the following equations:
  \begin{align*}
    Y(0)
  &= \beta_0 + \beta_1^T X_{\text{d}} + \beta_2^T X_{\text{im}} +
   \beta_3^T X_{\text{d}} \otimes X_{\text{d}} + \epsilon_0,
  \\ 
     Y(1)
  &= \beta_0 + \beta_1^T X_{\text{d}} + \beta_2^T X_{\text{im}} +
   \beta_3^T X_{\text{d}} \otimes X_{\text{d}} + 
  \beta_{t1}^T X_{\text{im}} + \beta_{t2}^T X_{\text{d}} \otimes X_{\text{im}}
  +\epsilon_1,
  \end{align*} 
where $ \epsilon_0$ ,$\epsilon_1 \sim N(0,0.1) $ , $X_{\text{d}}$ is the vector of demographic variables, $X_{\text{im}}$ is the vector of diagnosis categories and $\otimes$ is the symbol of Kronecker product. We assign values of $\beta$ in a clinically meaningful way and $\beta_3$ and $\beta_{t2}$ are sparse matrices in the sense that only a few variables would interact with each other. More details of the data generating process such as the exact values of $\beta$ could be found in our source code.

For treatment, we consider two main scenarios: observational study randomized study. The results of observational study setting are shown in Figure \ref{fig:4}, where treatment depends on some covariates: 
  \begin{align*}
   p(x_i) &= \frac{1}{ 1+ \exp(-\beta_t x_i) },  \\ 
T &\sim \text{Bern}\left(\frac{ p_2 p(x_i)}{  p(X) /N }\right),
  \end{align*} 
where $ X = (X_{\text{d}} , X_{\text{im}} ), p_2=\{0.1,0.2,0.3,0.4,0.5 \} $. The  $ \frac{ p_2 p(x_i)}{  p(X) /N } $  term is to control the mean of $p(x_i) $ to mimic an unbalanced assignment mechanism and $P(T=1|X)$ is called the propensity score. In randomized study setting, treatment does not depend on any covariates. Specifically, we set $T \sim \text{Bern}(p_1) $ where $ p_1=\{0.1,0.2,0.3,0.4,0.5\} $ to mimic an unbalanced assignment mechanism. The results of CATE estimation in this setting are presented in Figure \ref{fig:D1}. It turns out that the conclusions derived from this figure are very similar to the ones we state in Section \ref{sec:5.3}.

\subsubsection{Data Pre-processing Procedure of the STAR Dataset}
\label{appx:D.3.2}
The effect of class size on student's achievement is an important topic in the American K-12 education system. To study the effect, the State Department of Education in Tennessee conducted a four-year longitudinal, class-size randomized study called The Student/Teacher Achievement Ratio (STAR) from 1985 to 1989. Using the first graders' data from the STAR project, we estimate the effect of class size on class-level mathematical performance on a standardized test. Since the original dataset corresponds to a randomized study, the true effect of class size can be estimated directly. Following similar covariate selection as \citet{deshpandedeep}, we use highest degree obtained by teacher, career ladder position of teacher, number of years of experience of teacher, and teacher’s race as numerical features. To construct multi-modal covariates, the students' gender and ethnicity are replaced by images of corresponding characteristics from the UTK dataset \cite{zhang2017age}. All categorical covariates are converted into one-hot encoding. Test scores and all continuous covariates are normalized using min-max normalization. Observations with any missing values are filtered out, resulting in a total of 6563 data points. The train-validation-test split ratio is again set to be 40:20:40.

\begin{figure*}[t!]
\begin{minipage}[c]{0.68\textwidth}
    \includegraphics[width=\textwidth]{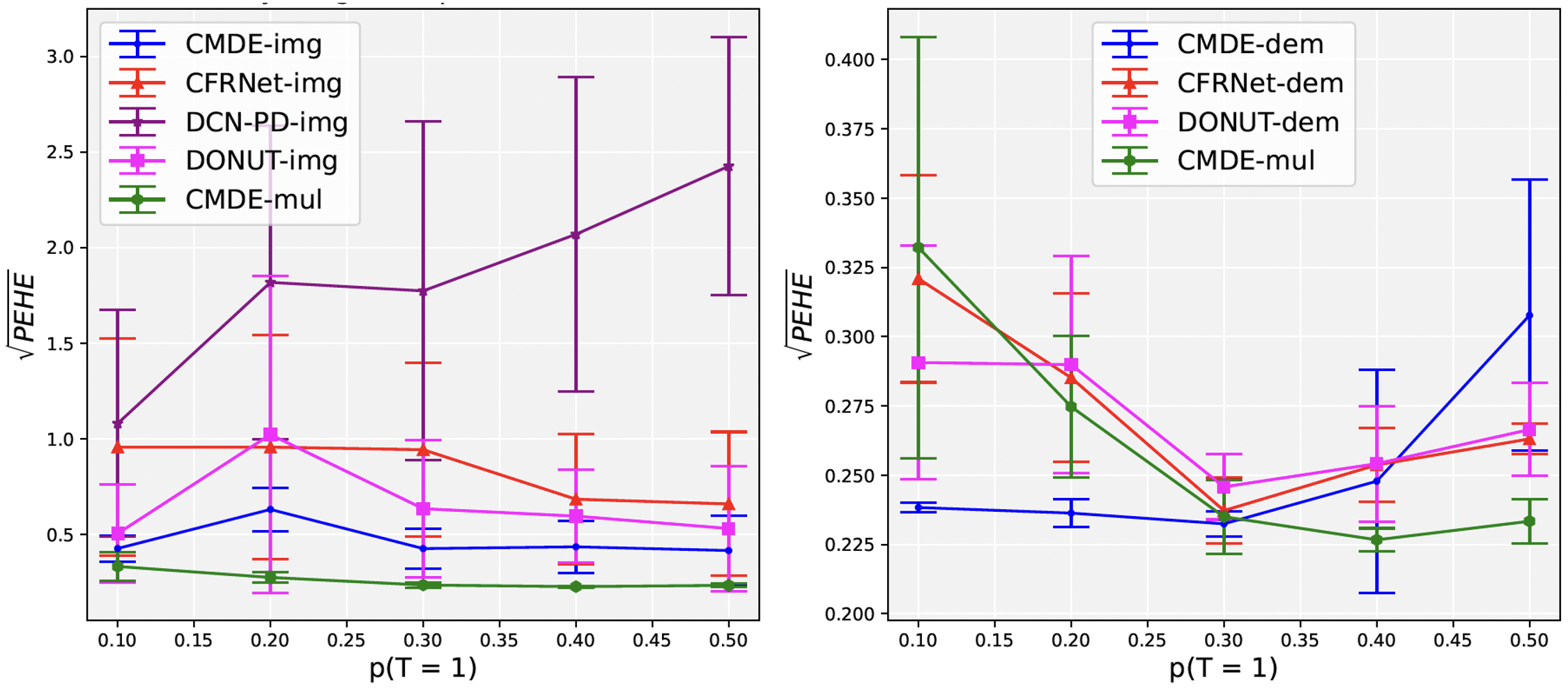}
\end{minipage}\hfill
\begin{minipage}[c]{0.30\textwidth}
    \caption{Results of CATE estimation $\left(\sqrt{\epsilon_{\text{PEHE}}}\right)$ on the semi-synthetic COVID-19 dataset in randomized study setting where the covariates are either X-ray images (left) or demographic information (right). CMDE with multi-modal covariates (both images and demographic information) are marked as CMDE-mul in both figures. The lines and error bars represent mean and half of the standard deviation of $\sqrt{\epsilon_{\text{PEHE}}}$, respectively. The error of DCN-PD with demographic information is too high so we do not show it in the right figure for better visualization.}
    \label{fig:D1}
\end{minipage}
\end{figure*}

\begin{table}[t!]
\centering
\resizebox{\textwidth}{!}{
\begin{tabular}{l|c|c|c|c} 
\hline
\multirow{2}{*}{} & \multicolumn{2}{c|}{COVID-19 dataset}      & \multicolumn{2}{c}{STAR dataset}            \\ 
\cline{2-5}
                  & Image covariates & Tabular data covariates & Image covariates & Tabular data covariates  \\ 
\hline
CMDE              & \begin{tabular}[c]{@{}c@{}}\#estimators $= 10$\\ICM kernel, depth $= 2$, \\\#channels in hidden \\layers $= 64$, \\filter size $= 3$, \\stride $= 1$ \\$\alpha_H = \alpha_T = 1$, \\$\alpha_{HT} = 0.5$, \\softplus activation ~\end{tabular}                 & \begin{tabular}[c]{@{}c@{}}\#estimators $= 10$\\ICM kernel, depth $= 2$, \\width $= 512$, \\$\alpha_H = \alpha_T = 1$, \\$\alpha_{HT} = 0.5$, \\softplus activation ~\end{tabular}                        & \begin{tabular}[c]{@{}c@{}}\#estimators $= 10$\\ICM kernel, depth $= 2$, \\\#channels in hidden \\layers $= 64$, \\filter size $= 3$, \\stride $= 1$ \\$\alpha_H = \alpha_T = 1$, \\$\alpha_{HT} = 0.1$, \\softplus activation ~\end{tabular}                 & \begin{tabular}[c]{@{}c@{}}\#estimators $= 10$\\ICM kernel, depth $= 2$, \\width $= 512$, \\$\alpha_H = \alpha_T = 1$, \\$\alpha_{HT} = 0.1$, \\softplus activation ~\end{tabular}                         \\ 
\hline
CFRNet            & \begin{tabular}[c]{@{}c@{}}depth ($\phi$ and $h$) $= 2$, \\\#channels in hidden \\layers of $\phi$ and $h = 64$,\\filter size $= 3$, \\stride $= 1$, \\$\alpha = 0.01$, \\ softplus activation ~\end{tabular}                 & \begin{tabular}[c]{@{}c@{}}depth ($\phi$ and $h$) $= 2$, \\width ($\phi$ and $h$) $ = 512$, \\$\alpha = 0.01$, \\softplus activation ~\end{tabular}                        & \begin{tabular}[c]{@{}c@{}}depth ($\phi$ and $h$) $= 2$, \\\#channels in hidden \\layers of $\phi$ and $h = 64$,\\filter size $= 3$, \\stride $= 1$, \\$\alpha = 0.01$, \\ softplus activation ~\end{tabular}                 & \begin{tabular}[c]{@{}c@{}}depth ($\phi$ and $h$) $= 2$, \\width ($\phi$ and $h$) $ = 512$, \\$\alpha = 0.01$, \\softplus activation ~\end{tabular}                         \\ 
\hline
DCN-PD            & \begin{tabular}[c]{@{}c@{}}depth (shared and \\idiosyncratic \\networks) $= 2$, \\\#channels in hidden layers \\of shared and \\idiosyncratic \\networks $ = 64$, \\filter size $= 3$, \\stride $= 1$, \\softplus activation ~\end{tabular}                 & \begin{tabular}[c]{@{}c@{}}depth (shared and \\ idiosyncratic \\networks) $= 2$, \\width (shared and \\idiosyncratic \\networks) $ = 512$, \\softplus activation ~\end{tabular}                        & \begin{tabular}[c]{@{}c@{}}depth (shared and \\idiosyncratic \\networks) $= 2$, \\\#channels in hidden layers \\of shared and \\idiosyncratic \\networks $ = 64$, \\filter size $= 3$, \\stride $= 1$, \\softplus activation ~\end{tabular}                 & \begin{tabular}[c]{@{}c@{}}depth (shared and \\ idiosyncratic \\networks) $= 2$, \\width (shared and \\idiosyncratic \\networks) $ = 512$, \\softplus activation ~\end{tabular}                         \\ 
\hline
DONUT             & \begin{tabular}[c]{@{}c@{}}depth ($\phi$ and $h$) $= 2$, \\\#channels in hidden layers \\of $\phi$ and $h = 64$, \\filter size $= 3$, \\stride $= 1$, \\softplus activation ~\end{tabular}                 & \begin{tabular}[c]{@{}c@{}}depth ($\phi$ and $h$) $= 2$, \\width ($\phi$ and $h$) $ = 512$, \\softplus activation ~\end{tabular}                        & \begin{tabular}[c]{@{}c@{}}depth ($\phi$ and $h$) $= 2$, \\\#channels in hidden layers \\of $\phi$ and $h = 64$, \\filter size $= 3$, \\stride $= 1$, \\softplus activation ~\end{tabular}                 & \begin{tabular}[c]{@{}c@{}}depth ($\phi$ and $h$) $= 2$, \\width ($\phi$ and $h$) $ = 512$, \\softplus activation ~\end{tabular}                         \\
\hline
\end{tabular}}
\caption{Model hyperparameters used for CMDE and other deep causal benchmark models in the COVID-19 X-ray image and STAR experiments}
\label{tab:3}
\end{table}
\newpage

\section{Discussion on Fine-tuning of Coefficients}
\renewcommand\thefigure{\thesection\arabic{figure}}
\setcounter{figure}{0}
\label{appx:E}

\begin{figure}[t!]
\centering
\begin{subfigure}[t]{0.315\textwidth}
    \includegraphics[width=\linewidth]{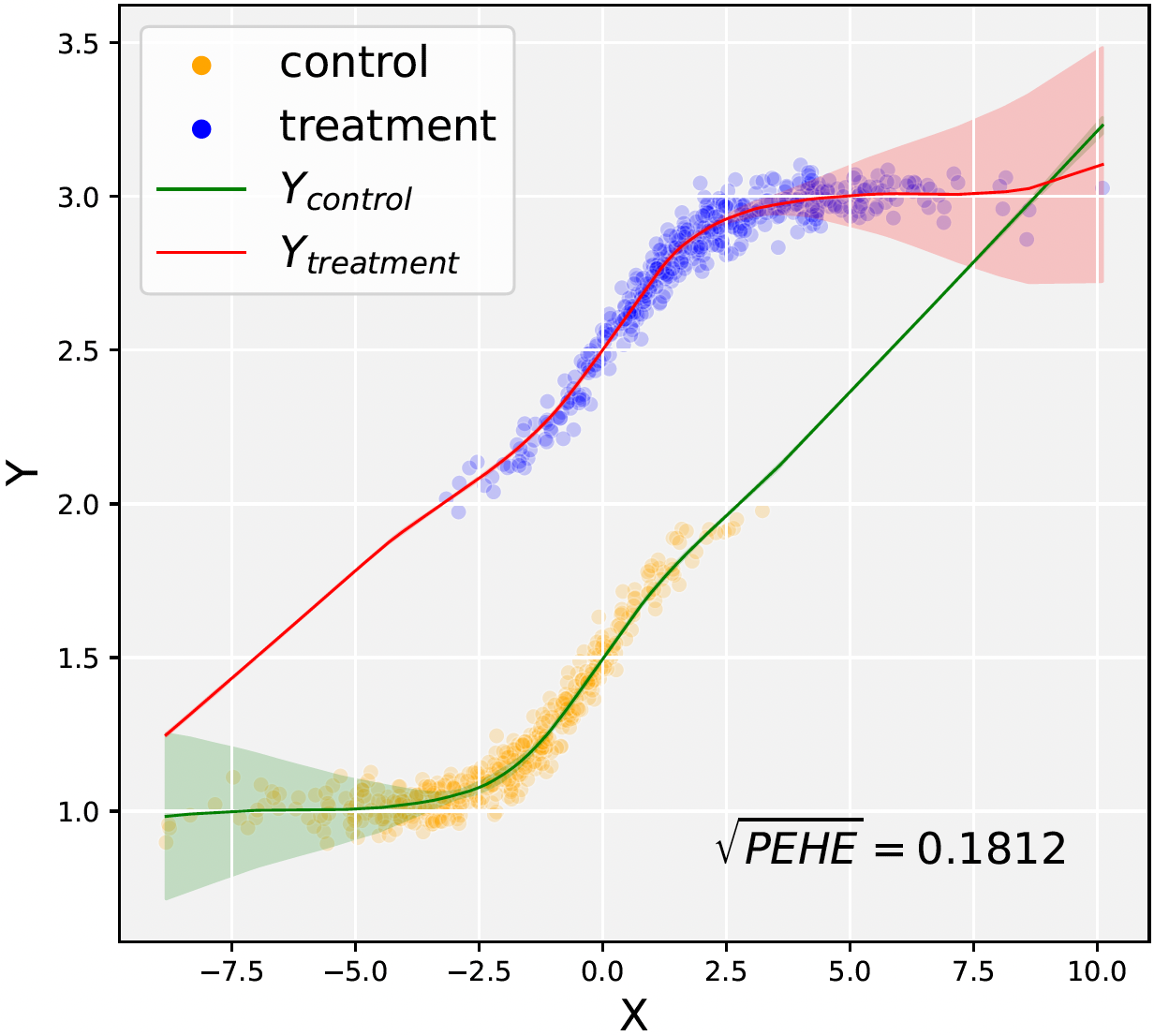}
    \label{fig:E1a}
\end{subfigure}
\begin{subfigure}[t]{0.315\textwidth}
    \includegraphics[width=\linewidth]{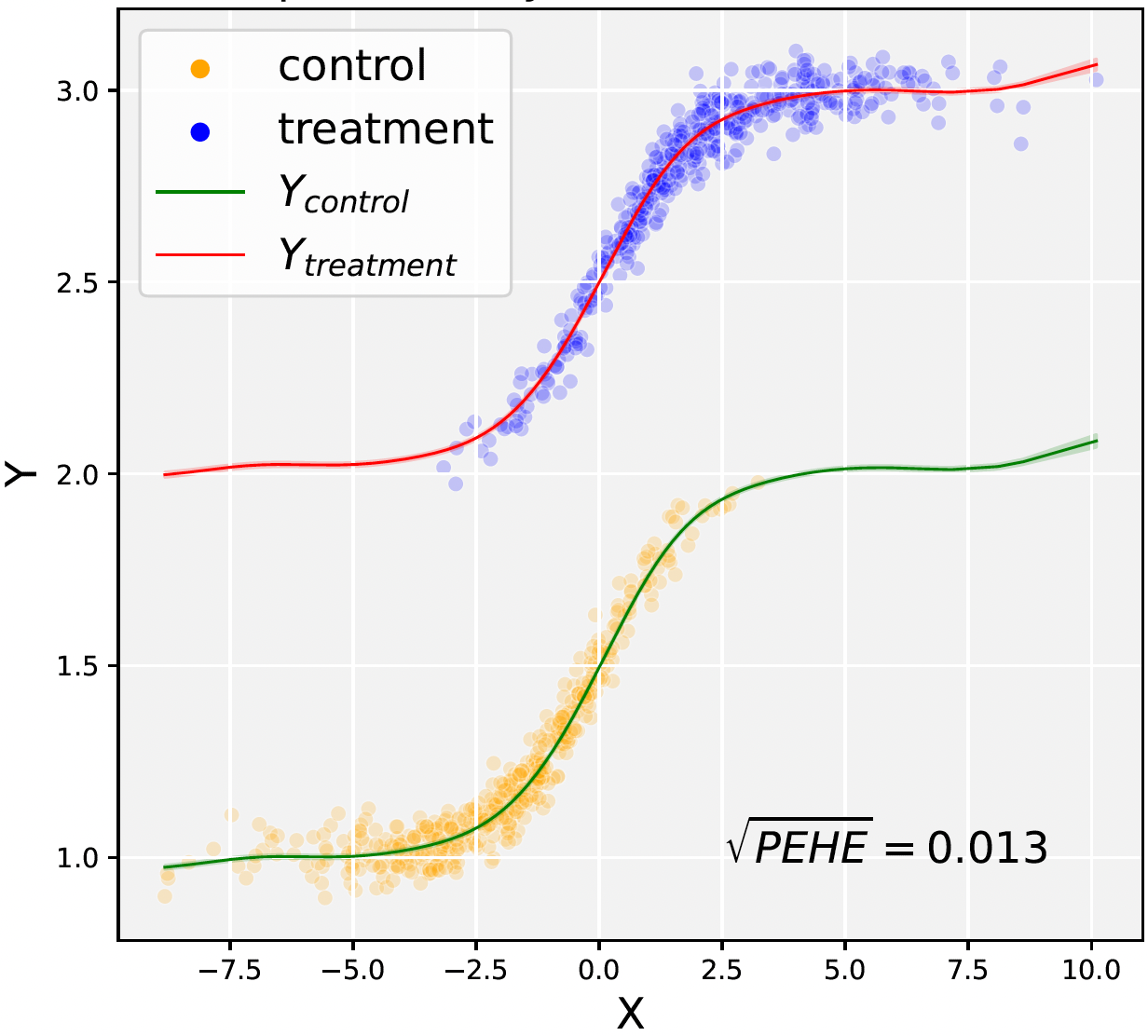}
    \label{fig:E1b}
\end{subfigure}
\begin{subfigure}[t]{0.315\textwidth}
    \includegraphics[width=\linewidth]{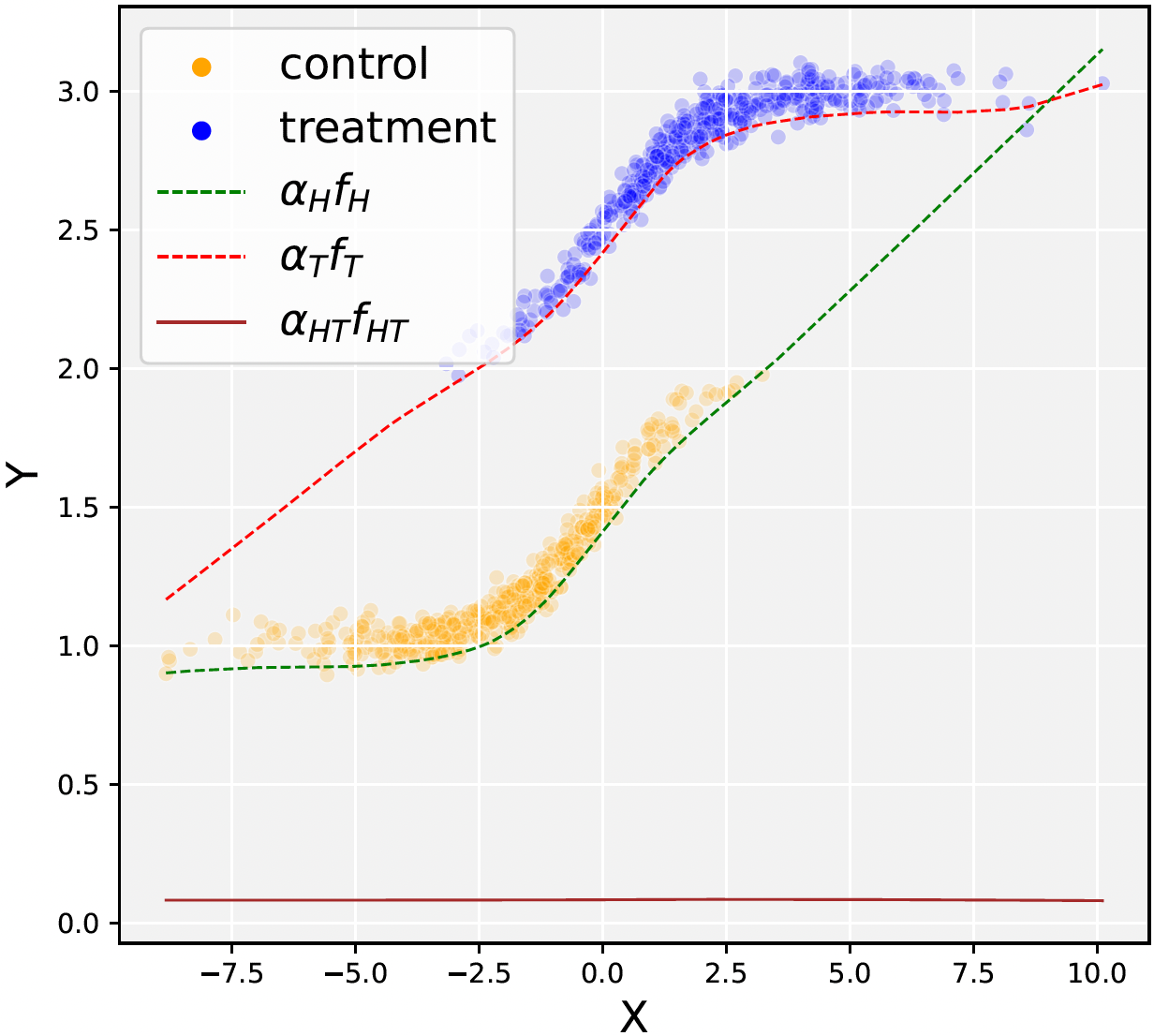}
    \label{fig:E1c}
\end{subfigure}
\vspace{-3mm}
\caption{Predictions for the control and treatment groups on the synthetic dataset by CMDE (left) with $\alpha_H = 1, \alpha_T = 1, \alpha_{HT} = 0$ and multi-task GP with ICM kernel (middle) where dots represent observed samples, lines represent mean predictions, and shaded regions represent predicted values within 2 standard deviations. In addition, we also plot the contribution of group-specific and shared components for CMDE (right). It can be observed that with inappropriate coefficient initializaiton, CMDE fails to capture the response surfaces of the control and treatment groups and yields a much larger error on treatment effect estimation than multi-task GP.}
\label{fig:E1}
\end{figure}

\textcolor{black}{As mentioned in the Discussion, CMDE is sensitive to the initial values of the coefficients applied in NNs (e.g. $\alpha_H$, $\alpha_T$, and $\alpha_{HT}$), thus requiring us to have some prior knowledge about which type of information, group-specific or shared, is more dominant. This is dependent on the task at hand. For example, in the synthetic data example in Section \ref{sec:5.1}, the shared information between the two groups (i.e, shape of the two curves) is important for estimating the potential outcomes when there is no group-specific information. Note that we refer to the group-specific information as the distinct features of the covariates $X$ in each group, and here the two curves are only off by a constant instead of anything related to the covariate $X$ (see Appendix D.1 for the data generation procedure). Therefore, we initialize the coefficients to be $\alpha_H = \alpha_T = 0$ and $\alpha_{HT} = 1$. To verify this point, we also repeat this experiment by setting $\alpha_H = \alpha_T = 1$ and $\alpha_{HT} = 0$. We realize that without correctly capturing the shape of the two curves, CMDE performs much worse on the treatment effect estimation as shown in Figure \ref{fig:E1}. On the contrary, the Twins dataset contains 40 covariates pertaining to pregnancy, twin births, and parents. The treatment is defined as $T = 1$ as being the heavier twin and $T = 0$ as being the lighter twin. The outcome is defined as the 1-year mortality rate. Here, our prior knowledge is that the unique features of each child will contribute more to the mortality rate, so we initialize the coefficients to be $\alpha_H = \alpha_T = 1$ and $\alpha_{HT} = 0.1$.} 

\section{Accessibility of the Datasets}
\renewcommand\thefigure{\thesection\arabic{figure}}
\setcounter{figure}{0}
\label{appx:F}
All datasets used in our experiments are available at \url{https://github.com/jzy95310/ICK/tree/main/data} and are released under the MIT license. For ACIC, Twins, Jobs, and STAR, the original datasets have an open-access license and are publicly available. For COVID-19 experiment, the original dataset is available at: \url{https://github.com/ieee8023/covid-chestxray-dataset} where each image has a specified license including Apache 2.0, CC BY-NC-SA 4.0, and CC BY 4.0. All other files and scripts are released under a CC BY-NC-SA 4.0 license.


\end{document}